\documentclass[Afour,sageh,times]{sagej}

\usepackage{moreverb,url}

\usepackage{algorithmic}
\usepackage{algorithm}

\theoremstyle{plain}

\newtheorem{theorem}{Theorem}
\newtheorem{corollary}{Corollary}

\newtheorem{remark}{Remark}

\usepackage[colorlinks,bookmarksopen,bookmarksnumbered,citecolor=red,urlcolor=red]{hyperref}

\newcommand\BibTeX{{\rmfamily B\kern-.05em \textsc{i\kern-.025em b}\kern-.08em
T\kern-.1667em\lower.7ex\hbox{E}\kern-.125emX}}

\def\volumeyear{2025}

\begin{document}

\runninghead{Sun, Ding, and Zhu}

\title{A Geometric Method for Base Parameter Analysis in Robot Inertia Identification Based on Projective Geometric Algebra}

\author{Guangzhen Sun\affilnum{1}, Ye Ding\affilnum{1}, and Xiangyang Zhu\affilnum{1}}

\affiliation{\affilnum{1}The State Key Laboratory of Mechanical System and
	Vibration, School of Mechanical Engineering, Shanghai Jiao Tong University,
	Shanghai 200240, China}

\corrauth{Ye Ding, The State Key Laboratory of Mechanical System and
	Vibration, School of Mechanical Engineering, Shanghai Jiao Tong University,
	Shanghai 200240, China. \email{y.ding@sjtu.edu.cn}}

\begin{abstract}
This paper proposes a novel geometric method for analytically determining the base inertial parameters of robotic systems. The rigid body dynamics is reformulated using projective geometric algebra, leading to a new identification model named ``tetrahedral-point (TP)" model. Based on the rigid body TP model, coefficients in the regresoor matrix of the identification model are derived in closed-form, exhibiting clear geometric interpretations. Building directly from the dynamic model, three foundational principles for base parameter analysis are proposed: the shared points principle, fixed points principle, and planar rotations principle. With these principles, algorithms are developed to automatically determine all the base parameters. The core algorithm, referred to as Dynamics Regressor Nullspace Generator (DRNG), achieves $O(1)$-complexity theoretically following an $O(N)$-complexity preprocessing stage, where $N$ is the number of rigid bodies. The proposed method and algorithms are validated across four robots: Puma560, Unitree Go2, a 2RRU-1RRS parallel kinematics mechanism (PKM), and a 2PRS-1PSR PKM. In all cases, the algorithms successfully identify the complete set of base parameters. Notably, the approach demonstrates high robustness and computational efficiency, particularly in the cases of PKMs. Through the comprehensive demonstrations, the method is shown to be general, robust, and efficient.
\end{abstract}


\maketitle

\section{Introduction}
Accurate dynamic modeling is essential for advanced motion planning and control in robotics, with inertial parameters serving as a fundamental component (\cite{ZHANG2025113208}, \cite{Lee2024ModernReview}, \cite{Huang2023TRO}, \cite{Han2020TRO}). These parameters, comprising mass, center of mass, and the rotational inertia tensor of each rigid body, characterize the physical behavior of the robotic system. Since the 1980s, a considerable amount of research has been centered around their identification (\cite{Lee2020Convex}, \cite{Yoshida2000PositiveDefinite}, \cite{Atkeson1986Estimation}). Although recent machine learning approaches enable implicit estimation of these parameters, explicit identification remains critical for reducing the sim-to-real gap and ensuring reliable performance in physical deployments.

A key observation in inertial parameter identification is that the robot's actuation is linear to its inertial parameters (\cite{KhalilModeling2002}). This property allows the dynamic model to be reformulated as group of linear equations, making the identification problem a linear regression problem. The group of linear equations is referred to as the identification model of the robot. In this formulation, the measured actuation signals constitute the target vector, the inertial parameters form the coefficient vector, and the regressor matrix is determined by the robot's motion. However, due to holonomic constraints imposed by the robot's joints, the regressor matrix is typically rank-deficient. Not all inertial parameters can be independently identified. The minimal subset of parameters that can be independently identified under sufficient excitation is known as the base parameters. Identifying this subset and understanding its structure is referred to as base parameter analysis in inertial parameter identification (\cite{Wensing2024Geometric}). It contributes to reduce the computational cost and enhance the robustness of identification process (\cite{Khalil1987MinimumOperationsTree}).

A variety of methods have been proposed to address the base parameter analysis problem, which are commonly classified into numerical, symbolic, and geometric methods.Numerical methods were among the earliest ones developed and applied in the literature because of its generality and no need for analytical analysis (\cite{Atkeson1986Estimation}). These methods operate on a finite set of randomly sampled data, which must be sufficiently exciting to ensure that the linear dependency in the regressor matrix is solely caused by the robot’s geometric structure. Using QR decomposition or singular value decomposition in matrix theory, a group of base parameters can be obtained numerically (\cite{Sheu1991IndependentSpace}, \cite{Gautier1991NumericalCal}, \cite{Gautier1990NumericalCalBase}). Because the numerical methods are inherently suitable to general linear regression problem, they can be applied to any type of robot. Especially, numerical methods are widely applied in the cases of parallel robots, also known as parallel kinematics mechanisms (PKMs), which normally have multiple closed loops and complex geometric structure (\cite{Danaei2017SphericalPKM}, \cite{Farhat2008PKM_RPS}, \cite{Goldenberg1992IdentClosedKinematicsChain}). However, despite its generality, numerical methods can not provide insight to the dynamics properties of the robot. Besides, the base parameter analysis can be unreliable in the presence of insufficiently exciting data and numerical issues.

Symbolic methods utilize the equations of motion (EoM) to identify unidentifiable parameters and regroup the linear dependent ones analytically (\cite{Khalil2014OpenSYMORO}, \cite{Mayeda1990BaseParamModel}). Early symbolic methods focused on tree structure robots. \cite{Khalil1987MinimumOperationsTree} conducted detailed analysis of the dynamic equations, and provided symbolic regrouping formulas. \cite{Gautier1990DirectCalSerial} further utilized the energy model so that the symbolic formulation procedure was simplified. Due to the complexity of EoM, multiple special cases can cause extra linear dependencies. As a result, several techniques have been developed to improve the generality of symbolic method for tree structure robots (\cite{Kawasaki1991MinimumTree}, \cite{Khalil1994DirectCalComment}). Nevertheless, closed loops are common in practical robot design. \cite{Bennis1990Parallelogram} considered the kinematics constraints caused by parallelogram closed loops, and they further generalized to the closed-loop robots (\cite{Khalil1995SymbolicClosedLoop}), but only rotational and prismatic joints are considered. For the PKMs with universal and spherical joints, \cite{Khalil2004PKM_GS} conducted detailed analysis on the model of Gough-Stewart robots and provided symbolic base parameters. Additionally, methods in \cite{Kawasaki1996SymbolicClosedChain} and \cite{Klodmann2015BaseParaComplex} directly operated on the regressor matrix, also resulting in reliable symbolic solutions. However, symbolic methods demand considerable effort to account for diverse special cases. Moreover, the analysis becomes intricate when multi-degree-of-freedom joints and complex closed-loop structures are involved, making these methods seldom used in the context of PKMs.

As the development of geometric methods in robotics (\cite{Park2018} \cite{Lynch2017MR} \cite{Featherstone2007}), they are also applied in inertial parameter identification (\cite{Lee2021OptimalExcitation} \cite{fu_lie-theory-based_2021}) for a coordinates-free solution.  \cite{Anigstein2001BaisPartI} and \cite{Anigstein2001Basis-free} utilized the vectorial expression of points and avoid the complex coordinates analysis. \cite{Chen2002Planar} proposed the regrouping rules of inertial parameters for planar mechanisms in geometric perspective, and they further developed their methods for general spatial mechanisms \cite{Chen2002GeneralSpatial}. \cite{Ros2012PKM3RPS} and \cite{Ros2015InertiaTransfer} introduced the concept of multipoles into the base parameter analysis problem, and applied inertia transfer to determine the base inertial parameters. It is noted the base parameters of a 3-RPS PKM were analyzed symbolically, which is a breakthrough in the literature. Although some geometric properties of the inertia are considered in these methods and multiple propositions were summarized, these methods still resulted in symbolic operations and meticulous analysis was required in practical applications. Besides the manipulators, multiple research has also been conducted on the identification of floating-base systems (\cite{Bonnet2018HumanoidFixForceSensor} \cite{Wensing2018LMI} \cite{Venture2009Realtime}) with the development of humanoid robots. \cite{Ayusawa2014LeggedMSD} proposed a joint-torque-free identification method for floating-base systems, where it is found that all the base parameters can be identified via base-link dynamics. It significantly simplifies the identification procedure, and motivates multiple methods (\cite{Bonnet2018HumanoidFixForceSensor}). Recently, \cite{Wensing2024Geometric} made a breakthrough in the base parameter analysis problem. They proposed a provably correct geometric method to determine all the base parameters based on inertia transfer and the underlying Lie group structure. Besides, they developed an $O(N)$ algorithms suitable for not only fixed-base, but also floating-base systems. This method almost closed the base parameter analysis problem. Nevertheless, only the closed kinematic loops caused by joint motors was discussed. The cases of complex systems with multiple closed loops, such as PKMs, can not be solved through the algorithm.

Despite significant progress in base parameter analysis methods, a key challenge remains: the lack of a clear geometric interpretation of the regressor matrix. Although certain geometric properties have been identified and exploited—for instance, the fact that additional mass distributed along a joint's rotational axis does not influence its actuation (\cite{Wensing2024Geometric})—such insights are not directly shown in the dynamic equations. A fundamental geometric property of inertial parameters is the physical feasibility (\cite{Lee2021OptimalExcitation} \cite{Lee2020Convex} \cite{Sousa2014LMI}), which requires the spatial inertial matrix to be positive definite and the rotational inertia tensor to satisfy the triangle inequalities. This condition is equivalent to enforcing the positive definiteness of the pseudo-inertial matrix (\cite{Wensing2018LMI}). However, these geometric constraints are often omitted during the formulation of identification models, resulting in a linear model in vector space. To address these issues, it is necessary to develop new mathematical frameworks for robot dynamics identification that preserve physical feasibility while offering a clearer geometric interpretation of the regressor matrix.

In this paper, we propose the framework of projective geometric algebra (PGA) (\cite{GunnThesis} \cite{Dorst}) to address these problems and developed new geometric method for base parameter analysis in inertial parameter identification. Geometric algebra (GA) is also known as Clifford algebra (\cite{Bayro-Corrochano2021104326}). It introduces ``geometric product" into vector space, and enriches the algebraic structure of vector algebra. At the same time, it is capable of modeling lines, planes, rigid body motions all in vector forms, which is effective in solving geometric problems. Since the 2000s, it has been applied to solve physical problems and well know for its intuitiveness and conciseness (\cite{Hestenes2001OldWine} \cite{Hestenes2003spacetime}). In recent years, increasing research has pushed forward its application in robotics, including kinematics modeling (e.g., \cite{Li2016Mobility}) (e.g., \cite{bayro2007differential}), stiffness analysis (e.g., \cite{Li2019Stiffness}), calibration (e.g., \cite{Sui2024Calibration}), dynamics modeling (e.g., \cite{Hadfield} \cite{bayro2019robot}), motion planning (e.g., \cite{CHEN2022MotiongPlanning}), and optimal control (e.g., \cite{low_geometric_2023}). Among the different types of GA, the PGA proposed by Gunn (\cite{gunn2017geometric} \cite{gunn2017geometric}) is the minimal geometric algebra to model Euclidean space in a coordinate-free way. It utilizes a subalgebra  isomorphism to dual quaternions for rigid body motions, and points, lines and planes are represented with vectors. It has shown the power to solve physical problems in geometric way (e.g., \cite{Li2025GeometricExact} \cite{brehmer2023geometric}). However, current GA approaches to modeling rigid body inertia largely retain the classical matrix-based interpretation, offering little conceptual departure from traditional methods. Consequently, its use in base parameter analysis has received limited attention in existing research.

In our previous work (\cite{sun_analytical_2023} \cite{sun_high-order_2023}), we proposed the four-mass-particle model for rigid body dynamics with PGA based on the mass particle dynamic model in \cite{Dorst}. In this paper, we extend this result by introducing the tetrahedral-point model. Furthermore, we propose three laws of base parameter analysis for base parameters, enabling the analytical solution of base parameters. Besides, we proposed new algorithms to implement the proposed method. We summarize the contributions of this papers as 
\begin{enumerate}
	\item We propose the tetrahedral-point model for inertial parameters identification. The model is based on dynamics model instead of energy model, and it provides a clear geometric interpretation of the dynamics regressor matrix with a simple formula.
	\item We propose three geometric principles for base parameter analysis in inertial parameters identification (shared points principle, fixed points principle and planar rotations principle). The principles are based on the dynamic model of the robot, instead of the energy model, and they are suitable for all kinds of robot, no matter fixed-base or floating-base, with or without closed loops.
	\item We develop a novel algorithm for dynamics regressor nullspace generator (DRNG) based on the three principles.  The algorithm is with theoretical $O(1)$-complexity after $O(N)$-complexity preprocessing algorithms, where $N$ is the number of rigid bodies contained in the robot. It is the best in literature, to the best knowledge of authors. 
\end{enumerate}
To evaluate the correctness of the proposed method, we present four demonstrations: the Puma560 serial manipulator, the Unitree G2 quadruped robot, the 2RRU-1RRS PKM, and the 2PRS-1PSR PKM. These examples cover a wide range of robotic structures, including open-chain manipulators, floating-base systems, and parallel robots with multiple closed kinematic loops. To the best of the authors’ knowledge, this represents the most comprehensive set of demonstrations currently available in the literature.

The paper is organized as follows. We fist introduce basic concepts and notations in the inertial parameter identification in Section \ref{sec: problem statement}. In Section \ref{sec: robot modeling}, we start with the dynamic model of a point mass with classical method, and introduce concepts in PGA to transform the model to the framework of PGA. Successively, we propose the ``tetrahedral-point (TP)" model of a rigid body and a robot system, and interpret the geometric meaning of the coefficients in regressor matrix. Furthermore, we propose three principles for base parameter analysis based on the TP model in Section \ref{sec: principles}, and develop the corresponding algorithms in Section \ref{sec: Algorithm}. The correctness of the proposed method and algorithms are evaluated in Section \ref{sec: demonstrations} through four examples. Finally, conclusions and future work are given in Section \ref{sec: conclusion}.

\section{Problem statement}
\label{sec: problem statement}
The dynamics of a robotic system can be formulated as 
\begin{equation}
	M(q)\ddot{q}+C(q,\dot{q}) + G(q) = \tau
\end{equation}
The fundamental fact is that the left hand side of the dynamic equation can be rearranged into a linear form about the inertial parameters
\begin{equation}
	\label{eq: Y pi = tau}
	Y(q,\dot{q}, \ddot{q}) \pi = \tau
\end{equation}
Then, the inertial parameters identification problem is summarized as a regression problem:
\begin{equation}
	\label{eq: identification}
	\min_{\pi} \frac{1}{2}\sum_{i=1}^{N_s}(Y_i \pi - \tau_i)^T(Y_i \pi - \tau_i)
\end{equation}
where $N_S$ is the number of samplings.

However, the regressor $Y(q,\dot{q}, \ddot{q})$ is rank-deficient because of the geometric constraints. It means that the parameters in $\pi$ can not be fully identified, causing infinite solutions and ill-condition problem of \eqref{eq: identification}. As a result, it is necessary to find out the base parameters that can be identified independently. This problem is referred as the base parameter analysis in inertial parameter identification.

Motivated from \cite{Wensing2024Geometric}, a method for the base parameter analysis is to find out all the basis of the null space of $Y(q,\dot{q}, \ddot{q})$ regardless of the value of $q,\dot{q}, \ddot{q}$.
\begin{equation}
	Null(Y) = \{\pi_{null}|Y(q,\dot{q}, \ddot{q}) \pi_{null} = 0, \forall q,\dot{q}, \ddot{q} \}
\end{equation}
Then, the space of base parameters is the complement space of $Null(Y)$.
\begin{equation}
	\label{eq: base space}
	B(Y) = Null(Y)^{\perp}
\end{equation}
In this paper, the subspace $Null(Y)$ is referred to as the dynamics regressor nullspace. The problem to be solved in this paper is how to determine its basis vectors given the geometric model of a robot.

Besides, the physical feasibility of a rigid body inertia is proved critical in identification (e.g. \cite{Sousa2014LMI} and \cite{Lee2020Convex}). For a rigid body, its spatial inertia is defined as a $6\times6$ matrix:
\begin{equation}
	\underline{I} = \left[\begin{array}{cc}
	\underline{J} & m [\underline{c}]\\
	m [\underline{c}]^T & m \underline{1}_3
	\end{array}\right]
\end{equation}
where $\underline{J}$ is the rotational inertia, $[\underline{c}]$ is the skew matrix of the center of mass, $m$ is the mass, and $\underline{1}_3$ is a 3rd-order identity matrix. It is found that the inertia can be transformed to a $4\times4$ matrix named as the pseudo-inertia matrix (\cite{Wensing2018LMI}):
\begin{equation}
	\underline{N} = \left[\begin{array}{cc}
	\underline{\Sigma}& m \underline{c}\\
	m \underline{c}^T & m
	\end{array}\right]
\end{equation}
where $\underline{\Sigma} = \frac{1}{2} tr(\underline{J}) \underline{1}_3 - \underline{J}$ is the second moments of inertia. The physical feasibility is to require that $\underline{N}$ is a symmetric positive-definite matrix, which can be formulates as:
\begin{equation}
	\underline{N} \in SDP(4)
\end{equation}
where $SDP(4)$ is the manifold of 4th-order Symmetric Positive-Definite (SDP) matrix. However, to the best of authors' knowledge, there is no method to formulate the identification model of robot dynamics directly with $\underline{N}$. This issue will also be addressed in this paper.
\section{Robot dynamics modeling with PGA}
\label{sec: robot modeling}
In this section, we propose the ``tetrahedral-point (TP)" model of rigid body dynamics, and apply it to model robots. 

A systematic and rigorous introduction to robot dynamics modeling in PGA has been proposed in our previous work (\cite{sun_analytical_2023} and \cite{sun_high-order_2023}). In order to explain more concisely why PGA should be introduced in robot dynamics modeling, we will focused on the most related concepts in the theory of PGA, and won't introduce them until they are needed.
\subsection{The dynamic model of a point mass}
\begin{figure}[t]
	\centering
	\includegraphics[width=1\linewidth]{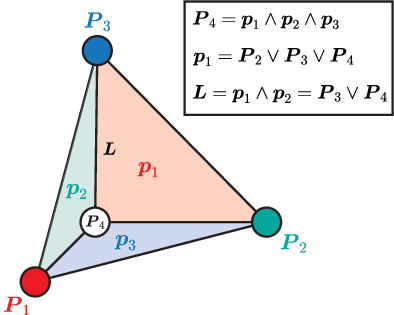}
	\caption{Illustration of the geometric meaning of ``$\wedge$" and "$\vee$". ``$\wedge$" implements the meet of two geometric elements, and ``$\vee$" implements the joint. The line $\boldsymbol{L}$ can be the meet of two planes or the join of two points.}
	\label{fig: wedge_vee}
\end{figure}
In classical mechanics, the dynamics of a point mass is determined by the linear and angular momentum theorem, which is formulated in coordinates-dependent form as
\begin{subequations}
	\label{eq: point mass dynamics}
	\begin{align}
		\underline{{ f}} &= m \underline{{ a}}\\
		\underline{{ \tau}} &= m \underline{{ p}} \times \underline{{a}}
	\end{align}
\end{subequations}
where $\underline{{f}}, \underline{{\tau}}\in \mathbb{R}^3$ are coordinates of the force and moment acting on the point, $\underline{{p}}, \underline{{a}} \in \mathbb{R}^3$ are coordinates of the position and acceleration of the point mass, and $m$ is the mass. Considering the Pl\"ucker coordinates of a line, the point mass dynamics can be reformulated as the following equations in a coordinates-free way:
\begin{equation}
	\label{eq: w=ma}
	\boldsymbol{w} = ma\ \operatorname*{Line}(\boldsymbol{P}, \ddot{\boldsymbol{P}})
\end{equation}
where $\boldsymbol{w}$ is the wrench acting on the point, $a=\|\underline{\boldsymbol{ a}}\|_{2}$ is the acceleration 2-norm, and $\operatorname*{Line}(\boldsymbol{P}, \ddot{\boldsymbol{P}})$ is the directed line passing $\boldsymbol{P}$ and directing in $\ddot{{\boldsymbol{P}}}$. Therefore, it is reasonable to introduce an algebra capable of modeling $\operatorname*{Line}(\boldsymbol{P}, \ddot{\boldsymbol{P}})$ in a coordinate-free way. Accordingly, we choose the projective geometric algebra $\mathbb{G}_{3,0,1}$.  

$\mathbb{G}_{3,0,1}$ is constructed based on the 4-dimensional dual vector space ${\mathbb{R}^{4}}^{*}$ to model the 3-dimensional Euclidean geometry. The ``dual" means that vectors are used to represent planes instead of points, which is the convention in the theory of projective geometry. Four orthogonal vectors $\boldsymbol{e}_i\ (i=0,1,2,3)$ can be chosen as the basis such that
\begin{subequations}
	\label{eq: inner product}
	\begin{align}
		\boldsymbol{e}_i \cdot \boldsymbol{e}_j &= 0,\ (i\neq j)\\
		\label{eq: self inner}\boldsymbol{e}_i \cdot \boldsymbol{e}_i &= \left\{ \begin{array}{l}
			1,\ i=1,2,3 \\
			0,\ i=0 
		\end{array}\right.
	\end{align}
\end{subequations}
In \eqref{eq: self inner}, it shows that the squared norms of 3, 0, and 1 basis vectors are 1, -1, and 0, respectively. Therefore, the subscript of $\mathbb{G}_{3,0,1}$ is $(3,0,1)$. 

The plane $ax+by+cz+d=0$ in Euclidean space is represented as the vector $\boldsymbol{p}=a\boldsymbol{e}_1 +b\boldsymbol{e}_2 + c\boldsymbol{e}_3 + d\boldsymbol{e}_0$ in $\mathbb{G}_{3,0,1}$. 
With the definition of planes, lines can be generated by the meet of two planes. The meet operator is implemented by the ``outer product ($\wedge$)", also named as the wedge product. It is defined as an asymmetric bilinear operator, such that
\begin{equation}
	\boldsymbol{p}_1 \wedge \boldsymbol{p}_2 = - \boldsymbol{p}_2 \wedge \boldsymbol{p}_1, \forall \boldsymbol{p}_1, \boldsymbol{p}_2 \in {\mathbb{R}^4}^*
\end{equation}
In the algebraic perspective, the outer product can be seen as the asymmetric part of the tensor product of two vectors. Like the tensor product, the outer product extends the vector space to linear spaces with higher grades. For example, $\boldsymbol{e}_1 \wedge \boldsymbol{e}_2$ produces a new basis vector with grade $2$, denoted as $\boldsymbol{e}_{12}$. In the case of $\mathbb{G}_{3,0,1}$, the linear space with grade $2,3,4$ can be generated, and they are called bivector space ${\bigwedge}^2{\mathbb{R}^4}^*$, trivector space ${\bigwedge}^3{\mathbb{R}^4}^*$ and pseudo scalar space ${\bigwedge}^4{\mathbb{R}^4}^*$, respectively. The basis vector of pseudo scalar space $\boldsymbol{e}_{0123}$ is also denoted as $\boldsymbol{I}$. Besides, the scalar space is treated as the space with grade $0$.

Because the outer product implement the meet of geometric elements, it turns out that bivectors represent lines (the meet of two planes) and trivectors represent points (the meet of three planes). Accordingly, a point $\boldsymbol{P}$ is represented in $\mathbb{G}_{3,0,1}$ as
\begin{equation}
	\label{eq: point coordinates}
	\boldsymbol{P} = x \boldsymbol{e}_{032} + y\boldsymbol{e}_{013} + z\boldsymbol{e}_{021} + \boldsymbol{e}_{123} \in {\bigwedge}^{3} {\mathbb{R}^4}^*
\end{equation}
where $[x,y,z,1]^T$ are the homogeneous coordinates. A line $\boldsymbol{L}$ is represented in $\mathbb{G}_{3,0,1}$ as
\begin{equation}
	\begin{split}
		\boldsymbol{L} =& d_x \boldsymbol{e}_{23} +d_y \boldsymbol{e}_{31} + d_z \boldsymbol{e}_{12} \\
		&+ p_x \boldsymbol{e}_{01} + p_y \boldsymbol{e}_{02} + p_z \boldsymbol{e}_{03} \in {\bigwedge}^{2} {\mathbb{R}^4}^*
	\end{split}
\end{equation}
where $[d_x,d_y,d_z, p_x, p_y, p_z]^T$ are the Pl\"ucker coordinates. As illustrated in Fig.\ref{fig: wedge_vee}, the line $\boldsymbol{L}$ can be written as the outer product of planes $\boldsymbol{p}_1$ and $\boldsymbol{p}_2$, and the point $\boldsymbol{P}_4$ can be written as the outer product of $\boldsymbol{p}_1$, $\boldsymbol{p}_2$, and $\boldsymbol{p}_3$

To this end, the motion of a point can be described as a time-parameterized curve in trivector space $\boldsymbol{P}(t)$. The acceleration $\ddot{\boldsymbol{P}}(t)$ is
\begin{equation}
	\label{eq: acceleration coordinates}
	\ddot{\boldsymbol{P}} = a_x \boldsymbol{e}_{032} + a_y\boldsymbol{e}_{013} + a_z\boldsymbol{e}_{021}  \in {\bigwedge}^{3} {\mathbb{R}^4}^*
\end{equation}
It is treated as an infinite point in projective space.
\begin{remark}
	The infinite points are common in projective geometry. They can be seen as the ``directions". For example, the velocity and acceleration of a point do not have ``positions" but have ``directions". Therefore, they are represented as infinite points in $\mathbb{G}_{3,0,1}$. Furthermore, there are also infinite lines and infinite plane in $\mathbb{G}_{3,0,1}$. The coordinates of an infinite line is in the form $[0,0,0, p_x, p_y, p_z]^T$. It can also indicates a direction $[p_x, p_y, p_z^T]$. For detailed discussion of them, readers are referred to \cite{sun_analytical_2023}.
\end{remark}

However, the line determined by $\boldsymbol{P}(t)$ and $\ddot{\boldsymbol{P}}(t)$ is still not defined since it requires the join operator instead of the meet operator. The vee product ``$\vee$" is defined to realize the join operator. Since only vee product of two points is involved in this paper, we provide a loose but practical definition of ``$\vee$". We first define the dual operator such that
\begin{subequations}
	\begin{align}
		\boldsymbol{e}_{032}^* = \boldsymbol{e}_1,\ \boldsymbol{e}_1^* = \boldsymbol{e}_{032}\\
		\boldsymbol{e}_{013}^* = \boldsymbol{e}_2,\ \boldsymbol{e}_2^* = \boldsymbol{e}_{013}\\
		\boldsymbol{e}_{021}^* = \boldsymbol{e}_3,\ \boldsymbol{e}_3^* = \boldsymbol{e}_{021}\\
		\boldsymbol{e}_{123}^* = \boldsymbol{e}_0,\ \boldsymbol{e}_0^* = \boldsymbol{e}_{123}
	\end{align}
\end{subequations}
The operator is linear, and the vee product of two trivectors is defined as
\begin{equation}
	\boldsymbol{P}_1 \vee \boldsymbol{P}_2 = {(\boldsymbol{P}_1^* \wedge \boldsymbol{P}_2^*)}^* , \forall \boldsymbol{P}_1, \boldsymbol{P}_2 \in {\bigwedge}^{3} {\mathbb{R}^4}^*
\end{equation}
To be more specific, if the coordinates of the two points are $[x_1, y_1, z_1, w_1]^T:=[\underline{p}_1^T, w_1]^T$ and $[x_2, y_2, z_2, w_2]^T:=[\underline{p}_2^T, w_2]^T$, the vee product returns
\begin{equation}
	\boldsymbol{L} =[\boldsymbol{e}_{23},\boldsymbol{e}_{31},\boldsymbol{e}_{12},\boldsymbol{e}_{01},\boldsymbol{e}_{02},\boldsymbol{e}_{03}] \left[\begin{array}{c} w_1\underline{p}_2 - w_2\underline{p}_1 \\
		\underline{p}_1 \times \underline{p}_2  \end{array}\right]
\end{equation}
where $\times$ is the cross product in vector algebra. It should be noted that if $w_i=1$, the point is called a Euclidean point, and if $w_i=0$, the point is called an infinite point.

Considering \eqref{eq: point mass dynamics} and \eqref{eq: acceleration coordinates}, the right hand side of \eqref{eq: w=ma} is formulated in $\mathbb{G}_{3,0,1}$ as
\begin{equation}
	\begin{split}
		&ma\ \operatorname*{Line}(\boldsymbol{P}, \ddot{\boldsymbol{P}}) \\
		=&[\boldsymbol{e}_{23},\boldsymbol{e}_{31},\boldsymbol{e}_{12},\boldsymbol{e}_{01},\boldsymbol{e}_{02},\boldsymbol{e}_{03}] \left[\begin{array}{c} m\underline{a} \\
			\underline{p} \times m\underline{a}  \end{array}\right]\\
		=& m \boldsymbol{P} \vee \ddot{\boldsymbol{P}}
	\end{split}
\end{equation}
Correspondingly, the wrench acting on the point mass is represented as a bivector:
\begin{equation}
	\label{eq: wrench}
	\begin{split}
		\boldsymbol{w}=&f_x\boldsymbol{e}_{23}+f_y\boldsymbol{e}_{31}+f_z\boldsymbol{e}_{12}\\
		&+\tau_x\boldsymbol{e}_{01}+\tau_y\boldsymbol{e}_{02}+\tau_z\boldsymbol{e}_{03}\in {\bigwedge}^{2} {\mathbb{R}^4}^*
	\end{split}
\end{equation}
In summary, the dynamics of a point mass can be formulated with PGA as
\begin{equation}
	\boldsymbol{w} = m \boldsymbol{P} \vee \ddot{\boldsymbol{P}}
\end{equation}
\subsection{The TP model of rigid body dynamics} 
The rigid body motion in PGA is represented as ``motors", and it acts on geometric elements through ``geometric product". The geometric product can be seen as the composition of inner product and outer product. Given two vectors $\boldsymbol{a}, \boldsymbol{b} \in {\mathbb{R}^4}^*$, their geometric product is defined as
\begin{equation}
	\boldsymbol{a} \boldsymbol{b}=\boldsymbol{a} \cdot \boldsymbol{b} + \boldsymbol{a} \wedge \boldsymbol{b}
\end{equation}
It is noted that the resultant element belongs to the direct sum space $\mathbb{R}\oplus {\mathbb{R}^4}^*$. Successively, the geometric product extends the original vector space ${\mathbb{R}^4}^*$ to the geometric algebra $\mathbb{G}_{3,0,1}$:
\begin{equation}
	\mathbb{G}_{3,0,1} = \mathbb{R}\oplus {\mathbb{R}^4}^* \oplus {\bigwedge}^2{\mathbb{R}^4}^* \oplus {\bigwedge}^3{\mathbb{R}^4}^* \oplus {\bigwedge}^4{\mathbb{R}^4}^* 
\end{equation}
Elements in $\mathbb{G}_{3,0,1}$ are linear combination of vectors with different grades, and they are called ``multivectors". Among the multivectors in $\mathbb{G}_{3,0,1}$, the ones only containing even-grade components can be applied to describe rigid body motion. They are denoted as $\boldsymbol{M}$ with coordinates:
\begin{equation}
	\begin{split}
		\boldsymbol{M}=& c_1 + c_2 \boldsymbol{e}_{23} + c_3 \boldsymbol{e}_{31} + c_4 \boldsymbol{e}_{12} \\
		&+ c_5 \boldsymbol{e}_{01} + c_6 \boldsymbol{e}_{02} + c_7 \boldsymbol{e}_{03} + c_8 \boldsymbol{I} \in \mathbb{G}_{3,0,1}^+
	\end{split}
\end{equation}
Its ``reverse" is defined as
\begin{equation}
	\begin{split}
		\widetilde{\boldsymbol{M}}=& c_1 - c_2 \boldsymbol{e}_{23} - c_3 \boldsymbol{e}_{31} - c_4 \boldsymbol{e}_{12} \\
		&- c_5 \boldsymbol{e}_{01} - c_6 \boldsymbol{e}_{02} - c_7 \boldsymbol{e}_{03} + c_8 \boldsymbol{I}
	\end{split}
\end{equation}
The Lie group of rigid body motion is formally defined as
\begin{equation}
	\boldsymbol{\mathcal{M}}_{3,0,1}=\{\boldsymbol{M}\in \mathbb{G}_{3,0,1}^+ | \widetilde{\boldsymbol{M}} \boldsymbol{M} = 1\}
\end{equation}
This Lie group is called the motor group, and its elements are motors. The identity element is $1$, and the group product is the geometric product.

Given an arbitrary element $\boldsymbol{X}^0 \in \mathbb{G}_{3,0,1}$, which can be a point, a line, or a plane, its configuration after a rigid body motion $\boldsymbol{M}$ is
\begin{equation}
	\boldsymbol{X} = \widetilde{\boldsymbol{M}} \boldsymbol{X}^0 \boldsymbol{M}
\end{equation}
Take time derivatives on both sides, the velocity of $\boldsymbol{X}$ is
\begin{equation}
	\dot{\boldsymbol{X}}=\frac{\boldsymbol{X}\boldsymbol{V} - \boldsymbol{V}\boldsymbol{X}}{2} = \boldsymbol{X} \times \boldsymbol{V}
\end{equation}
where ``$\times$" is the cross product in $\mathbb{G}_{3,0,1}$, and $\boldsymbol{V}$ is the spatial velocity of the rigid body motion. It is defined as
\begin{equation}
	\boldsymbol{V}=2\widetilde{\boldsymbol{M}}\dot{\boldsymbol{M}} \in {\bigwedge}^2 {\mathbb{R}^4}^*
\end{equation}
It is also noted that the coordinates of $\boldsymbol{V}$ are exactly the coordinates of a spatial twist (\cite{Lynch2017MR}). That is
\begin{equation}
	\begin{split}
		\boldsymbol{V} =& \omega_x \boldsymbol{e}_{23} + \omega_y \boldsymbol{e}_{31} + \omega_z \boldsymbol{e}_{12}\\
		&+ v_x \boldsymbol{e}_{01} + v_y \boldsymbol{e}_{02} + v_z \boldsymbol{e}_{03}
	\end{split}
\end{equation}
where the definition of $[\omega_x, \omega_y, \omega_z, v_x, v_y, v_z]^T$ is the same as the velocity twist written in the fixed spatial frame. 
Accordingly, the Lie algebra $\boldsymbol{m}_{3,0,1}$ of $\boldsymbol{\mathcal{M}}_{3,0,1}$ is exactly the bivector space. Furthermore, with the definition of wrench in \eqref{eq: wrench}, a duality can be defined through the outer product:
\begin{equation}
	\label{eq: duality}
	\boldsymbol{w} \wedge \boldsymbol{V} = \boldsymbol{V} \wedge \boldsymbol{w} = P \boldsymbol{I} \in {\bigwedge}^4{\mathbb{R}^4}^* \cong \mathbb{R}
\end{equation}
where $P\in \mathbb{R}$ is the power of the wrench $\boldsymbol{w}$. In the following part, we will omit $\boldsymbol{I}$ in $P \boldsymbol{I}$. Therefore, the wrench $\boldsymbol{w}$ belongs to the dual Lie algebra $\boldsymbol{m}_{3,0,1}^*$, which is also identical to bivector space.

Furthermore, the acceleration of $\boldsymbol{X}$ is
\begin{equation}
	\ddot{\boldsymbol{X}} = \dot{\boldsymbol{X}} \times \boldsymbol{V} + \boldsymbol{X} \times \dot{\boldsymbol{V}}
\end{equation}

A rigid body is defined as a set of points, denoted as $\mathcal{B}$, that all points follow the same rigid body motion. Given the configuration $\boldsymbol{M}$, velocity $\boldsymbol{V}$ and acceleration $\dot{\boldsymbol{V}}$ of the rigid body, the motion of points $\boldsymbol{P}$ contained in $\mathcal{B}$ is
\begin{align}
	\label{eq: position}\boldsymbol{P} &= \widetilde{\boldsymbol{M}} \boldsymbol{P}^0 \boldsymbol{M}\\
	\label{eq: vel}\dot{\boldsymbol{P}} &= \boldsymbol{P} \times \boldsymbol{V}\\
	\label{eq: acc}\ddot{\boldsymbol{P}} &= \boldsymbol{P} \times \dot{\boldsymbol{V}} + \dot{\boldsymbol{P}} \times \boldsymbol{V}
\end{align}
where $\boldsymbol{P}^0$ is the initial position of the point when $\boldsymbol{M}=1$. With the density distribution function $\rho(\boldsymbol{P}): \mathcal{B} \mapsto \mathbb{R}$, the mass of the rigid body is calculated by integration:
\begin{equation}
	m =\int_{\mathcal{B}} {\rm d}m= \int_{\mathcal{B}} \rho(\boldsymbol{P}) {\rm d} \boldsymbol{P}
\end{equation}
where ${\rm d} \boldsymbol{P}$ is the volume element, and ${\rm d}m$ is the differential mass. Therefore, the dynamics of the rigid body can be formulated as
\begin{equation}
	\begin{split}
		\boldsymbol{w}=\int_{\mathcal{B}} {\rm d} m \left(\boldsymbol{P} \vee \ddot{\boldsymbol{P}}\right) 
	\end{split}
\end{equation}
Suppose four points $\boldsymbol{E}_1, \boldsymbol{E}_2, \boldsymbol{E}_3, \boldsymbol{E}_4$ are chosen as the basis of trivector space. Then, the above equation can be expanded w.r.t. this basis as
\begin{equation}
	\begin{split}
		\boldsymbol{w}=& \int_{\mathcal{B}} {\rm d} m \left[\left(\sum_{m=1}^{4} c_m\boldsymbol{E}_m\right) \vee \left(\sum_{n=1}^{4} c_n\ddot{\boldsymbol{E}}_n\right)\right]\\
		=& \sum_{m=1}^{4}\sum_{n=1}^{4} \left( \int_{\mathcal{B}} {\rm d} m c_m c_n  \right) \left(\boldsymbol{E}_m \vee \ddot{\boldsymbol{E}}_n\right)\\
		=& \sum_{m=1}^{4}\sum_{n=1}^{4} N_{mn} \dot{\boldsymbol{\Pi}}_{mn}
	\end{split}
\end{equation}
From the definition, it can be proved that $N_{mn}=N_{nm}$. Furthermore, the above equation can be rearranged so that the inertial parameters can be characterized by a 4-th order symmetric positive-definite matrix:
\begin{equation}
	\label{eq: TP model rigid body}
	\boldsymbol{w} = tr(N^T \underline{\dot{\boldsymbol{\Pi}}}) := \underline{\dot{\boldsymbol{\Pi}}} * N = N * \underline{\dot{\boldsymbol{\Pi}}}
\end{equation}
where $\underline{N}$ is a 4-th order square matrix whose element in the m-th row and the n-th column is $N_{mn}$, and $\underline{\dot{\boldsymbol{\Pi}}}$ is also a 4-th order square matrix whose element in the m-th row and the n-th column is a bivector $\dot{\boldsymbol{\Pi}}$. This equation of rigid body dynamics was first proposed in \cite{sun_analytical_2023}. It was applied in robot dynamics computation, and has not been discussed in the field of inertial parameters identification. It is noted that the key to obtain \eqref{eq: TP model rigid body} is to choose four basis trivectors of the trivector space. We name these four basis vectors $\boldsymbol{E}_1, \boldsymbol{E}_2, \boldsymbol{E}_3, \boldsymbol{E}_4$ as tetrahedral points of a rigid body, because the four points can be seen as the vertices of a tetrahedron. For example, the basis trivectors can be chosen as illustrated in Fig.\ref{fig: TP points}a. They are four tetrahedron vertices, and the edges $OP_x$, $OP_y$, and $OP_z$ are orthogonal to each other. The rigid body dynamics model \eqref{eq: TP model rigid body} is named as the tetrahedral-point (TP) model of a rigid body.

\begin{figure}[t]
	\centering
	\includegraphics[width=1\linewidth]{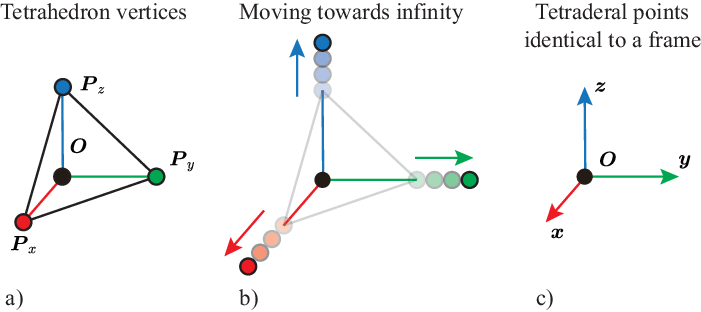}
	\caption{Illustration of the ``tetrahedral points". The four tetrahedron vertices in a) can be chosen as a set of tetrahedral points. As shown in b), move $\boldsymbol{P}_x, \boldsymbol{P}_y, \boldsymbol{P}_z$ towards the infinity, and they will converge to three infinite points $\boldsymbol{x}, \boldsymbol{y}, \boldsymbol{z}$, which are denoted as arrows in c).}
	\label{fig: TP points}
\end{figure}

We propose a natural choice of the tetrahedral points. As illustrated in Fig.\ref{fig: TP points}b, move the points $\boldsymbol{P}_x$, $\boldsymbol{P}_y$, and $\boldsymbol{P}_z$ towards the infinity. In the theory of projective geometry, they will terminate at three infinite points. In Fig.\ref{fig: TP points}c, the three infinite points are denoted with arrows. As a result, the tetrahedral points are composed of a point $\boldsymbol{O}$ and three directions $\boldsymbol{x}, \boldsymbol{y}, \boldsymbol{z}$, which turns out to be a frame. In robotics, it is common to attach a frame to a rigid body, and describe its motion with homogeneous matrices. The initial configuration of the rigid body can be represented by a constant homogeneous matrix ${\underline{T}_{b}^s}^0$:
\begin{equation}
	{\underline{T}_{b}^s}^0= \left[\begin{array}{cc} 
		\underline{R}^s_b &\underline{p}^s\\
		\underline{0} & 1\end{array}\right] 
	= \left[\begin{array}{cccc}r_{11} & r_{12} & r_{13} & p_1\\
		r_{21} & r_{22} & r_{23} & p_2\\
		r_{31} & r_{32} & r_{33} & p_3\\
		0 & 0 & 0 & 1\end{array}\right]
\end{equation}
where $\underline{R}^s_b$ is the rotation matrix of the body-fixed frame with respect to (w.r.t.) the fixed spatial frame, and $\underline{p}^s$ is the coordinates of the origin of the body-fixed frame. We found that each column of ${\underline{T}_{b}^s}^0$ can be seen as the coordinates of a point in $\mathbb{G}_{3,0,1}$. For example, as illustrated in Fig.\ref{fig: body fixed frame}, the coordinates of points $\boldsymbol{x}^0, \boldsymbol{y}^0, \boldsymbol{z}^0, \boldsymbol{O}^0$ w.r.t. the spatial fixed frame compose the homogeneous matrix ${\underline{T}_{b}^s}^0=[\underline{x}^0, \underline{y}^0, \underline{z}^0, \underline{O}^0]$.  It implies a natural choice of the tetrahedral points:
\begin{subequations}
	\label{eq: initial TP}
	\begin{align}
		\boldsymbol{E}_1^0 =& r_{11} \boldsymbol{e}_{032} + r_{21}\boldsymbol{e}_{013} + r_{31}\boldsymbol{e}_{021}\\
		\boldsymbol{E}_2^0 =& r_{13} \boldsymbol{e}_{032} + r_{22}\boldsymbol{e}_{013} + r_{32}\boldsymbol{e}_{021}\\
		\boldsymbol{E}_3^0 =& r_{13} \boldsymbol{e}_{032} + r_{23}\boldsymbol{e}_{013} + r_{33}\boldsymbol{e}_{021}\\
		\boldsymbol{E}_4^0 =& p_{1} \boldsymbol{e}_{032} + p_{2}\boldsymbol{e}_{013} + p_{3}\boldsymbol{e}_{021} + \boldsymbol{e}_{123}
	\end{align}
\end{subequations}
where the basis trivectors $\boldsymbol{e}_{032}, \boldsymbol{e}_{013}, \boldsymbol{e}_{021}, \boldsymbol{e}_{123}$ can be seen as the x-, y-, z-axis and the origin of the spatial fixed frame as illustrated in Fig.\ref{fig: body fixed frame}, respectively. The above formulation provides the initial configuration of the tetrahedral points when the configuration of the rigid body $\boldsymbol{M}=1$. During the rigid body motion, their positions, velocities and accelerations are updated according to \eqref{eq: position}, \eqref{eq: vel}, and \eqref{eq: acc}. In order to keep consistent with classical method, we apply the following denotation:
\begin{subequations}
	\begin{align}
		\boldsymbol{x} = \boldsymbol{E}_1\\
		\boldsymbol{y} = \boldsymbol{E}_2\\
		\boldsymbol{z} = \boldsymbol{E}_3\\
		\boldsymbol{O} = \boldsymbol{E}_4
	\end{align}
\end{subequations}
where $\boldsymbol{E}_i$ are points defined in \eqref{eq: initial TP}. As a result, any point $\boldsymbol{P}$ in the Euclidean space can be written as
\begin{equation}
	\boldsymbol{P} = x\boldsymbol{x} + y\boldsymbol{y} + z\boldsymbol{z} + \boldsymbol{O}
\end{equation}
Accordingly, the inertial matrix $\underline{N}$ turns out to be
\begin{equation}
	\begin{split}
		\underline{N} =& \int_{\mathcal{B}}  {\rm d}m \left[\begin{array}{c}
			x\\
			y\\
			z\\
			1\end{array}\right] \left[\begin{array}{c}
			x\\
			y\\
			z\\
			1\end{array}\right]^T\\
		=& \left[\begin{array}[]{cccc}
			MXX & MXY & MXZ & MX\\
			MXY & MYY & MYZ & MY\\
			MXZ & MYZ & MZZ & MZ\\
			MX & MY & MZ & M
		\end{array}\right]\\
		=&\left[\begin{array}[]{cc}
			\underline{\Sigma} & \underline{h}\\
			\underline{h}^T & m
		\end{array}\right]
	\end{split}
\end{equation}
It is exactly the pseudo inertial matrix (\cite{Wensing2018LMI}), where $m$ is the mass of the rigid body, $\underline{h}$ is the first moment of inertia and $\underline{\Sigma}$ is the second moments of inertia.

\begin{figure}[t]
	\centering
	\includegraphics[width=0.9\linewidth]{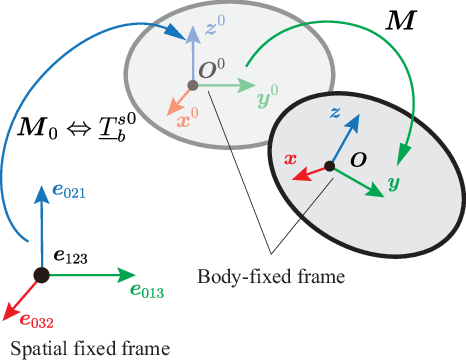}
	\caption{The TP model of a rigid body. The spatial fixed frame and the body-fixed frames can be interpreted as sets of tetrahedral points in $\mathbb{G}_{3,0,1}$. The initial configuration of the rigid body can be described with a motor $\boldsymbol{M}_0$, which is equivalent to ${\underline{T}_{b}^s}^0$ as explained in \eqref{eq: equivlent M T}. The configuration of the body-fixed tetrahedral points is determined by the motor $\boldsymbol{M}$. For example, $\boldsymbol{O} = \widetilde{\boldsymbol{M}} \boldsymbol{O}^0 \boldsymbol{M}$.}
	\label{fig: body fixed frame}
\end{figure}

\begin{remark}
	It is noted that ${\underline{T}_{b}^s}^0$ can also be seen as a rigid body motion from the fixed spatial frame to the body-fixed frame. Therefore, it is identical to a motor $\boldsymbol{M}_0$ in $\mathbb{G}_{3,0,1}$. Each column of ${\underline{T}_{b}^s}^0$ turns out to be the coordinates of $\boldsymbol{e}_{032}, \boldsymbol{e}_{013}, \boldsymbol{e}_{021}, \boldsymbol{e}_{123}$ after the rigid body motion $\boldsymbol{M}_0$. That is
	\begin{equation}
		\label{eq: equivlent M T}
		\begin{split}
			{\underline{T}_{b}^s}^0 = &\left[\widetilde{\boldsymbol{M}}_0 \boldsymbol{e}_{032} \boldsymbol{M}_0, \widetilde{\boldsymbol{M}}_0 \boldsymbol{e}_{013} \boldsymbol{M}_0, \right. \\
			& \left. \widetilde{\boldsymbol{M}}_0 \boldsymbol{e}_{021} \boldsymbol{M}_0, \widetilde{\boldsymbol{M}}_0 \boldsymbol{e}_{123} \boldsymbol{M}_0 \right]
		\end{split}
	\end{equation}
\end{remark}
\subsection{The TP model of robot dynamics} 
\label{sec: rigid body TP}
A robot can be treated as a multibody system composed of rigid bodies and holonomic constrained joints (flexible and soft robots are out of the scope of this paper). Its configuration is characterized by independent generalized coordinates $\underline{q}_a$. According to \eqref{eq: duality} and the theory of Lagrange mechanics, the robot dynamics can be formulated in variational form as
\begin{equation}
	\sum_{i=1}^{N_B}\left(\boldsymbol{w}_i - \underline{\dot{\boldsymbol{\Pi}}}_i * \underline{N}_i\right) \wedge \delta\boldsymbol{m}_i=0
\end{equation}
where $N_B$ is the number of rigid bodies, and $\boldsymbol{w}_i$ is the total wrench acting on body $i$. Besides, $\delta \boldsymbol{m} = 2 \widetilde{\boldsymbol{M}} \delta \boldsymbol{M} \in \boldsymbol{m}_{3,0,1}$ is the Lie group variation of the rigid body configuration $\boldsymbol{M}$. Considering the kinematics constraints, the variation $\delta\boldsymbol{m}_i$ can be formulated as:
\begin{equation}
	\delta\boldsymbol{m}_i = \sum_{j=1}^{N_A}\boldsymbol{L}^i_j \delta {q_a}_j
\end{equation}
where $N_A$ is the number of independent generalized coordinates. $\boldsymbol{L}^i_j \in \boldsymbol{m}_{3,0,1}$ is called the Jacobian line. As a result, the dynamics of the robot are formulated as $N_A$ ordinary differential equations:
\begin{equation}
	\tau_j = \sum_{i=1}^{N_B} \underline{R}_{j,i} * \underline{N}_i
\end{equation}
where $j=1,2,...,N_A$. The left hand side is called the generalized forces $\tau_j=\sum_{i=1}^{N_B} \boldsymbol{w}_i \wedge \boldsymbol{L}_j^i$. Generally, the generalized forces are the actuation forces / torques. On the right hand side, $\underline{R}_{j,i}$ is a $4\times 4$ real number valued matrix. The element on the m-th row and the n-th column is calculated by
\begin{equation}
	\label{eq: basic coeff}
	R_{j,i,m,n} = \boldsymbol{L}_{j}^{i} \wedge (\boldsymbol{E}^{i}_{m} \vee \ddot{\boldsymbol{E}}^{i}_{n})
\end{equation}
where $j=1,2,...N_A$ are the indices of actuators, $i=1,...,N_B$ are the indices of rigid bodies, and $m,n=1,2,3,4$ are the indices of tetrahedral points. We call $R_{j,i,m,n}$ as a dynamics regressor coefficient.

We formally name the matrix $\underline{R}_{j,i}$ as the dynamics regressor of body i w.r.t. the generalized coordinate j. With proper concatenation, we can obtain a $N_A \times N_B \times 4 \times 4 $ tensor $\underline{\mathcal{R}}$ composed of $\underline{R}_{j,i}$, and a $N_B \times 4 \times 4$ tensor $\underline{\mathcal{N}}$ composed of $\underline{N}_i$. All the inertial parameters of the robot are contained in $\underline{\mathcal{N}}$, and the generalized forces can be calculated by tensor product and contraction:
\begin{equation}
	\tau_j = \sum_{i,m,n} \mathcal{R}_{j,i,m,n} \mathcal{N}_{i,m,n}
\end{equation}
Therefore, we name the tensor $\underline{\mathcal{R}}$ as the dynamics regressor of the robot, and the above dynamic model as the tetrahedral-point (TP) model of the robot dynamics.

\section{Principles of base parameter analysis}
\label{sec: principles}
According to \eqref{eq: basic coeff}, we conclude that the dependency and unidentifiablity of inertial parameters are basically caused by the geometric relations between points ($\boldsymbol{E}_m^i$ and $\ddot{\boldsymbol{E}}_n^i$) and lines ($\boldsymbol{L}_j^i$). Based on this observation, we propose three principles for base parameter analysis in inertial parameters identification.

\subsection{Principle 1: shared points}
For two rigid bodies connected by a holonomic constrained joint, the constraints can be formulated in variation forms as
\begin{equation}
	\begin{split}
		\delta \boldsymbol{m}_{i} =& \delta\boldsymbol{m}_{\lambda(i)} + \boldsymbol{L}_{k}\delta q_{k}\\
		=&\delta\boldsymbol{m}_{\lambda(i)} + \sum_{j=1}^{N_A} \boldsymbol{L}_{k} \frac{\partial q_{k}}{\partial q_j}\delta q_{j}
	\end{split}
\end{equation}
where $\boldsymbol{L}_{k}$ is the constraint lines determined by the joint $k$. $\lambda(i)$ returns the index of the body $i$'s parent through the joint.
Therefore, the Jacobian lines have the relation
\begin{equation}
	\label{eq: jacobian line}
	\boldsymbol{L}^i_j = \boldsymbol{L}^{\lambda(i)}_j + \boldsymbol{L}_{k} \frac{\partial q_{k}}{\partial q_j}
\end{equation}
\begin{remark}
	If the joint $k$ is with more than 1 DOF, $\boldsymbol{L}_k \delta q_{k}$ is short for $\sum_{d=1}^{K_{DOF}}\boldsymbol{L}_{k_d} \delta q_{k_d}$, and so is $\boldsymbol{L}_{k} \frac{\partial q_{k}}{\partial q_j}$. For example, two constraint lines $\boldsymbol{L}_{k_1}$ and $\boldsymbol{L}_{k_2}$ exist in a universal joint ($K_{DOF}=2$). 
\end{remark}

Five common joints are considered in this paper, including revolute (R), prismatic (P), universal (U) , spherical (S), floating (F) joints. In the perspective of projective geometry, at least one point is shared by the two rigid bodies connected by these joints except F joints, and the shared points are located on the line $\boldsymbol{L}_{k}$. Suppose one shared point is $\boldsymbol{P}_{\alpha}$:
\begin{equation}
	\boldsymbol{P}_{\alpha}^i \equiv \boldsymbol{P}_{\alpha}^{\lambda(i)}
\end{equation}
The shared point $\boldsymbol{P}_{\alpha}$ can be written as the linear combination of the tetrahedral points of both rigid body $i$ and $\lambda(i)$.
\begin{equation}
	\label{eq: shared coordinates}
	\begin{split}
		\boldsymbol{P}_{\alpha} =& \sum_{t=1}^{4}c_{\alpha,t}^{i} \boldsymbol{E}^i_{t} = \sum_{t=1}^{4}c_{\alpha,t}^{\lambda(i)} \boldsymbol{E}^{\lambda(i)}_{t}\\
		=& \underline{\boldsymbol{E}}^i \underline{c_{\alpha}^{i}}  = \underline{\boldsymbol{E}}^{\lambda(i)} \underline{c_{\alpha}^{\lambda(i)}}
	\end{split}
\end{equation}
where $\underline{\boldsymbol{E}}^i = [\boldsymbol{E}^i_{1}, \boldsymbol{E}^i_{2}, \boldsymbol{E}^i_{3}, \boldsymbol{E}^i_{4}]$ is a row array collecting all the tetrahedral points, and $ \underline{c_{\alpha}^{i}}$ is a column array collecting all the coefficients.
As a result, we can obtain the following equality,
\begin{equation}
	\label{eq: coeff dependency}
	\begin{split}
		\boldsymbol{L}_{j}^{i} \wedge (\boldsymbol{P}_{\alpha} \vee \ddot{\boldsymbol{P}}_{\alpha}) =& \boldsymbol{L}_{j}^{\lambda(i)} \wedge (\boldsymbol{P}_{\alpha} \vee \ddot{\boldsymbol{P}}_{\alpha})\\
		&+ \frac{\partial q_{k}}{\partial q_j}\boldsymbol{L}_{k}\wedge (\boldsymbol{P}_{\alpha} \vee \ddot{\boldsymbol{P}}_{\alpha})\\
		=&\boldsymbol{L}_{j}^{\lambda(i)} \wedge (\boldsymbol{P}_{\alpha} \vee \ddot{\boldsymbol{P}}_{\alpha})
	\end{split}
\end{equation}
The second equality is caused by the following theorem and its corollary.

\begin{theorem}
	\label{thm: lines on plane}
	Two lines (simple bivectors \footnote{Simple bivectors means the bivector can be rewritten as the outer product of two vectors. In $\mathbb{G}_{3,0,1}$, not all bivectors are simple.})  $\boldsymbol{L}_1$ and $\boldsymbol{L}_2$ lie in the same plane if and only if $\boldsymbol{L}_1 \wedge \boldsymbol{L}_2 = 0$
\end{theorem}
\begin{proof}
	Since $\boldsymbol{L}_1$ and $\boldsymbol{L}_2$ are simple bivectors, they can be rewritten as the outer product of two vectors (two planes).
	
	If $\boldsymbol{L}_1$ and $\boldsymbol{L}_2$ lie in the same plane $\boldsymbol{a}$, then
	\begin{subequations}
		\label{eq: co planes}
		\begin{align}
			\boldsymbol{L}_1 &= \boldsymbol{a} \wedge \boldsymbol{b}_1\\
			\boldsymbol{L}_2 &= \boldsymbol{a} \wedge \boldsymbol{b}_2
		\end{align}
	\end{subequations}
	It results in $\boldsymbol{L}_1 \wedge \boldsymbol{L}_2 = \boldsymbol{a}\wedge\boldsymbol{b}_1 \wedge \boldsymbol{a} \wedge \boldsymbol{b}_2 = \boldsymbol{0}$ because of the property of $\wedge$.
	
	On the other hand, $\boldsymbol{L}_1 \wedge \boldsymbol{L}_2 = \boldsymbol{0}$ means $\boldsymbol{L}_1$ and $\boldsymbol{L}_2$ contain linear dependent vectors. It results in \eqref{eq: co planes}, which means the two lines lie in the same plane $\boldsymbol{a}$
\end{proof}

\begin{corollary}
	If a point $\boldsymbol{P}$ (Euclidean or infinite) is located on the line $\boldsymbol{L}$, then the following equation holds for any point $\boldsymbol{Q}$.
	\begin{equation}
		\label{eq: points on line}
		\boldsymbol{L} \wedge (\boldsymbol{P}\vee\boldsymbol{Q})=\boldsymbol{0}
	\end{equation}
\end{corollary}
\begin{proof}
	The line $\boldsymbol{P} \vee \boldsymbol{Q}$ joining two points intersect to $\boldsymbol{L}$ at $\boldsymbol{P}$, which means they lie in the same plane.
\end{proof}

\begin{corollary}
	\label{coro: infinite lines}
	Given any infinite line $\boldsymbol{L}$ and any infinite points $\boldsymbol{P}$, the equation \eqref{eq: points on line} holds for any infinite points $\boldsymbol{Q}$.
\end{corollary}
\begin{proof}
	All the infinite lines and points lie on the infinite plane $\boldsymbol{e}_0$. According to Theorem \ref{thm: lines on plane} and the fact that the joining line of two infinite points is an infinite line, the equation \eqref{eq: points on line} holds.
\end{proof}

Because points $\boldsymbol{P}_{\alpha}$ is located on the line $\boldsymbol{L}_{k}$, the term $\boldsymbol{L}_{k}\wedge (\boldsymbol{P}_{\alpha} \vee \ddot{\boldsymbol{P}}_{\alpha})$ holds to be zero. Besides, since the point is rigidly attached to both rigid bodies, its coordinates w.r.t. the both bodies remains constant. Then, by substituting \eqref{eq: shared coordinates} into \eqref{eq: coeff dependency}, we got the following linear dependency of dynamics regressor coefficients.
\begin{equation}
	\begin{split}
		&\sum_{t_r=1}^{4}\sum_{t_c=1}^{4}c_{\alpha,t_r}^{i} c_{\alpha,t_c}^{i} \left[\boldsymbol{L}_{j}^{i} \wedge ( \boldsymbol{E}^i_{t_r} \vee  \ddot{\boldsymbol{E}}^i_{t_c})\right] \\
		=& \sum_{t_r=1}^{4}\sum_{t_c=1}^{4}c_{\alpha,t_r}^{\lambda(i)} c_{\alpha,t_c}^{\lambda(i)} \left[\boldsymbol{L}_{j}^{\lambda(i)} \wedge ( \boldsymbol{E}^{\lambda(i)}_{t_r} \vee  \ddot{\boldsymbol{E}}^{\lambda(i)}_{t_c})\right]
	\end{split}
\end{equation}
Considering the definition of dynamics regressor coefficients, the equation is identical to
\begin{equation}
	\label{eq: group rule 1}
	\begin{split}
		&\sum_{t_r, t_c} c_{\alpha,t_r}^{i} c_{\alpha,t_c}^{i} R_{j,i,t_r,t_c} \\
		= &\sum_{t_r, t_c} c_{\alpha,t_r}^{\lambda(i)} c_{\alpha,t_c}^{\lambda(i)} R_{j,\lambda(i),t_r,t_c}
	\end{split}
\end{equation}
Then, it can be rearranged into tensor form as
\begin{equation}
	\label{eq: null space}
	\underline{R}_{j,i} * (\underline{c}_{\alpha}^{i} {\underline{c}_{\alpha}^{i}}^T) - \underline{R}_{j,\lambda(i)} * (\underline{c}_{\alpha}^{\lambda(i)} {\underline{c}_{\alpha}^{\lambda(i)}}^T) = 0
\end{equation}
Since the equation holds regardless of the value of $j$, it provides a basis of $Null(Y)$.

Nevertheless, more than one point is shared by the two rigid bodies in R and P joint. For two different shared points $\boldsymbol{P}_{\alpha_1}$ and $\boldsymbol{P}_{\alpha_2}$, the following equations can be obtained according to the above procedure:
\begin{subequations}
	\label{eq: null space cross}
	\begin{align}
		\underline{R}_{j,i} * (\underline{c}_{\alpha_1}^{i} {\underline{c}_{\alpha_2}^{i}}^T) - \underline{R}_{j,\lambda(i)} * (\underline{c}_{\alpha_1}^{\lambda(i)} {\underline{c}_{\alpha_2}^{\lambda(i)}}^T) = 0\\
		\underline{R}_{j,i} * (\underline{c}_{\alpha_2}^{i} {\underline{c}_{\alpha_1}^{i}}^T) - \underline{R}_{j,\lambda(i)} * (\underline{c}_{\alpha_2}^{\lambda(i)} {\underline{c}_{\alpha_1}^{\lambda(i)}}^T) = 0
	\end{align}
\end{subequations}
However, it is noted that $\underline{c}_{\alpha_1}^{i} {\underline{c}_{\alpha_2}^{i}}^T$ is not symmetric. In order to obtain basis in $SDP(4)$, \eqref{eq: null space cross} is transformed to the following equation:
\begin{equation}
	\begin{split}
		&\underline{R}_{j,i} * (\underline{c}_{\alpha_1}^{i} {\underline{c}_{\alpha_2}^{i}}^T + \underline{c}_{\alpha_2}^{i} {\underline{c}_{\alpha_1}^{i}}^T) \\
		- &\underline{R}_{j,\lambda(i)} * (\underline{c}_{\alpha_1}^{\lambda(i)} {\underline{c}_{\alpha_2}^{\lambda(i)}}^T + \underline{c}_{\alpha_2}^{\lambda(i)} {\underline{c}_{\alpha_1}^{\lambda(i)}}^T) = 0
	\end{split}
\end{equation}
Again, the equation holds regardless of the value of $j$. Therefore, it provides only one basis of $Null(Y)$. The shared points of the five different joint types are summarized in Table. \ref{tab: shared points}, and illustrated in Fig. \ref{fig: shared points}a.

\begin{remark}
	Results similar to the shared point principle were proposed in \cite{Ros2015InertiaTransfer}, where the concept of multipoles was applied to generate shared points. In this paper, the proposed shared points principle provides a geometric method, instead of symbolic method, to address the base parameter analysis problem.
\end{remark}

\begin{figure*}[t]
	\centering
	\includegraphics[width=1\linewidth]{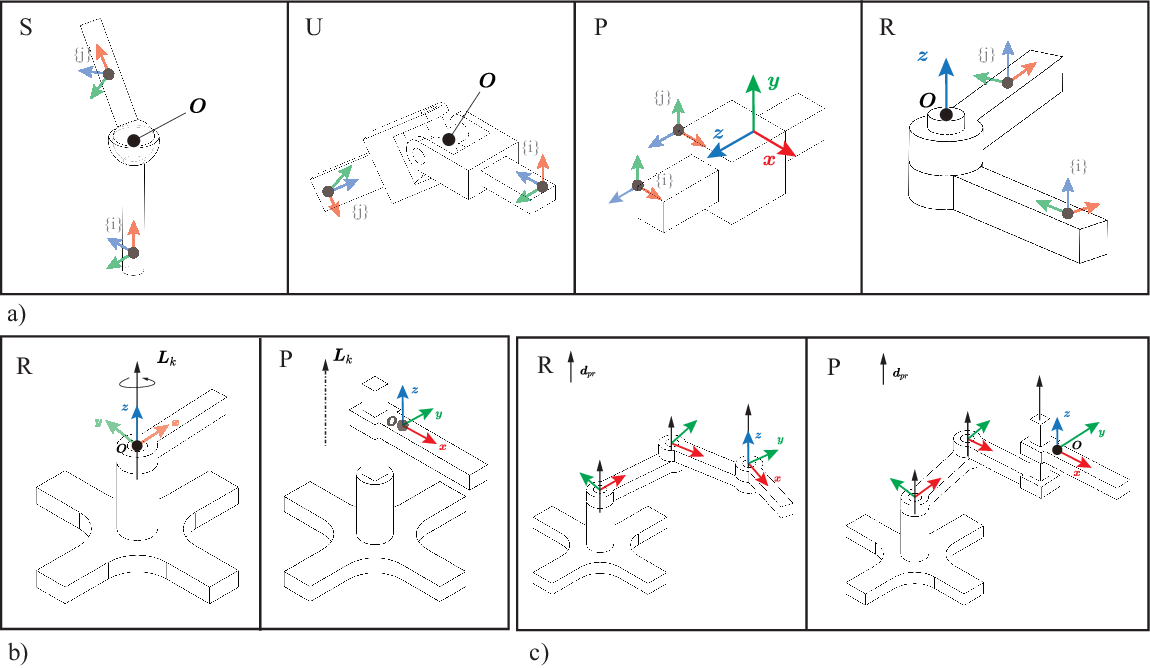}
	\caption{Illustration of the proposed three principles: a) Shared points in S, U, P, and R joints. The body-fixed frame of two rigid bodies are denoted as \{i\} and \{j\}. The shared points are corresponding to Table \ref{tab: shared points}. b). The fixed points in R (left) and P (right) joints. The points $\boldsymbol{O}$ and $\boldsymbol{z}$ on the R joint are fixed to the base. The directions of $\boldsymbol{x}, \boldsymbol{y}, \boldsymbol{z}$ on the P joint remains static, which means that they are fixed to the base in the perspective of PGA. c). Two typical cases when planar rotations happen. Left: multiple rigid bodies are connected through R joints with parallel rotational axis. Right: a rigid body is connected to a planar-rotation body through a P joint, and the translational direction coincide with the rotational axis.}
	\label{fig: shared points}
\end{figure*}

\begin{table}[h]
	\centering
	\begin{tabular}{ccc}
		\toprule
		Type  & Shared points & Description\\ \noalign{\smallskip}
		\midrule
		F  & $[\ ]$ & No shared points \\ \noalign{\smallskip}
		S  & $[\boldsymbol{O}]$ & The center of rotation\\ \noalign{\smallskip}
		U  & $[\boldsymbol{O}]$ & The intersection point of two axis\\ \noalign{\smallskip}
		P  & $[\boldsymbol{x}, \boldsymbol{y}, \boldsymbol{z}]$ & $\boldsymbol{z}$ is the translational direction\\ \noalign{\smallskip}
		R  & $[\boldsymbol{O}, \boldsymbol{z}]$ & Both on the rotational axis\\ \noalign{\smallskip}
		\bottomrule
	\end{tabular}
	\caption{The shared points of different types of joints.}
	\label{tab: shared points}
\end{table}

\subsection{Principle 2: fixed points}
For robotic manipulators attached to a fixed base, there exists shared points attached to the fixed base. We found that it is the main cause of the unidentifiability of inertial parameters, and it introduces extra linear dependency.

Suppose a rigid body $i$ is constrained to the fixed base by a joint with type R, P, U, or S. It means that the rigid body \{j\} illustrated in Fig.\ref{fig: shared points}a is fixed. Then, at least one shared point $\boldsymbol{P}_{\alpha}$ is attached to the fixed base. It means that 
\begin{equation}
	\label{eq: P acceleration 0}
	\ddot{\boldsymbol{P}}_{\alpha}\equiv \boldsymbol{0}
\end{equation}
According to \eqref{eq: jacobian line}, the Jacobian line of body $i$ is then
\begin{equation}
	\boldsymbol{L}^i_j = \boldsymbol{L}_{k} \frac{\partial q_{k}}{\partial q_j}
\end{equation}
Then, it is easy to prove that the following equations are satisfied.
\begin{align}
	\label{eq: zero accelerate}\boldsymbol{L}_j^i \wedge (\boldsymbol{P} \vee \ddot{\boldsymbol{P}}_{\alpha}) =0\\
	\boldsymbol{L}_j^i \wedge (\boldsymbol{P}_{\alpha} \vee \ddot{\boldsymbol{P}}) =0
\end{align}
where $\boldsymbol{P}$ is any points in the projective space, i.e. an arbitrary trivector. The first equation is caused by \eqref{eq: P acceleration 0}, and the second equation is caused by the fact that $\boldsymbol{P}_{\alpha}$ is located on $\boldsymbol{L}_{k}$ or $\boldsymbol{P}_{\alpha}$ and $\boldsymbol{L}_j^i$ are both infinite elements.
As a result, the following tensor-form equations are satisfied.
\begin{align}
	\underline{R}_{j,i} * (\underline{c}^{i} {\underline{c}_{\alpha}^{i}}^T) = 0\\
	\underline{R}_{j,i} * (\underline{c}_{\alpha}^{i} {\underline{c}^{i}}^T) = 0
\end{align}
where $\underline{c}^{i}$ and $\underline{c}_{\alpha}^{i}$ are coordinates of points $\boldsymbol{P}$ and $\boldsymbol{P}_{\alpha}$, respectively. It introduce more linear dependency apart from the ones caused by shared points.

However, the gravity introduces an extra acceleration $\boldsymbol{G}$ to all the rigid bodies. As a result, all the points attached to rigid bodies obtain an extra acceleration, which is calculated by
\begin{equation}
	\ddot{\boldsymbol{P}}^{extra} = \boldsymbol{P} \times \boldsymbol{G}
\end{equation}
It should be noted that $\boldsymbol{G}$ can be formulated as follows once the basis bivectors are chosen.
\begin{equation}
	\boldsymbol{G} = G_{x} \boldsymbol{e}_{01} + G_{y} \boldsymbol{e}_{02} + G_{z} \boldsymbol{e}_{03}
\end{equation}
Then, for all infinite points, $\ddot{\boldsymbol{P}}^{extra}$ is $\boldsymbol{0}$, while for Euclidean points, it is
\begin{equation}
	\ddot{\boldsymbol{P}}^{extra} = G_{x} \boldsymbol{e}_{032} + G_{y} \boldsymbol{e}_{013} + G_{z} \boldsymbol{e}_{021}
\end{equation}
Because of this extra acceleration, \eqref{eq: P acceleration 0} is broken when $\boldsymbol{P}_{\alpha}$ is a Euclidean point. It further affects the linear dependency among inertial parameters.

To be more specific, we analyze the case of S, U, P, and R joints in the following part. The denotation of fixed points are the same as shared points illustrated in Fig.\ref{fig: shared points}a, and their coordinates $\underline{c}_{x}^i, \underline{c}_{y}^i, \underline{c}_{z}^i, \underline{c}_{O}^i$ w.r.t. \{i\} can always be obtained.

\subsubsection{S/U joints}
\

For both S and U joints, only one point $\boldsymbol{P}_{\alpha}=\boldsymbol{O}$ is attached to the fixed base. Then, we can obtained the following equations apart from the shared points equations:
\begin{subequations}
	\label{eq: U/S fixed}
	\begin{align}
		\underline{R}_{j,i} * (\underline{c}_{O}^{i} {\underline{c}_{x}^{i}}^T + \underline{c}_{x}^{i} {\underline{c}_{O}^{i}}^T) = 0\\
		\underline{R}_{j,i} * (\underline{c}_{O}^{i} {\underline{c}_{y}^{i}}^T + \underline{c}_{y}^{i} {\underline{c}_{O}^{i}}^T) = 0\\
		\underline{R}_{j,i} * (\underline{c}_{O}^{i} {\underline{c}_{z}^{i}}^T + \underline{c}_{z}^{i} {\underline{c}_{O}^{i}}^T) = 0
	\end{align}
\end{subequations}
If the gravity $\boldsymbol{G}\neq \boldsymbol{0}$, then $\ddot{\boldsymbol{O}}\neq \boldsymbol{0}$. It breaks \eqref{eq: zero accelerate}, and makes \eqref{eq: U/S fixed} not satisfied any more. In this case, the fixed points will not introduce extra linear dependency of inertial parameters.
\subsubsection{P joints}
\

For P joints, as illustrated in the right figure of Fig.\ref{fig: shared points}b, $\boldsymbol{x}$, $\boldsymbol{y}$ and $\boldsymbol{z}$ are located on the fixed base, which means the direction wouldn't change. Following the same procedure, the following extra equations are satisfied
\begin{subequations}
	\label{eq: P fixed}
	\begin{align}
		\underline{R}_{j,i} * (\underline{c}_{O}^{i} {\underline{c}_{x}^{i}}^T + \underline{c}_{x}^{i} {\underline{c}_{O}^{i}}^T) = 0\\
		\underline{R}_{j,i} * (\underline{c}_{O}^{i} {\underline{c}_{y}^{i}}^T + \underline{c}_{y}^{i} {\underline{c}_{O}^{i}}^T) = 0\\
		\underline{R}_{j,i} * (\underline{c}_{O}^{i} {\underline{c}_{z}^{i}}^T + \underline{c}_{z}^{i} {\underline{c}_{O}^{i}}^T) = 0
	\end{align}
\end{subequations}
Since $\boldsymbol{x}$, $\boldsymbol{y}$ and $\boldsymbol{z}$ are all infinite points, \eqref{eq: P fixed} always holds even if $\boldsymbol{G}\neq \boldsymbol{0}$. 

\subsubsection{R joints}
\

For R joints, as illustrated in the left figure of Fig.\ref{fig: shared points}b, $\boldsymbol{O}$ and $\boldsymbol{z}$ are located on the fixed base. Accordingly, the following extra equations are satisfied
\begin{subequations}
	\label{eq: R fixed}
	\begin{align}
		\label{eq: Ox} \underline{R}_{j,i} * (\underline{c}_{O}^{i} {\underline{c}_{x}^{i}}^T + \underline{c}_{x}^{i} {\underline{c}_{O}^{i}}^T) = 0\\
		\label{eq: Oy} \underline{R}_{j,i} * (\underline{c}_{O}^{i} {\underline{c}_{y}^{i}}^T + \underline{c}_{y}^{i} {\underline{c}_{O}^{i}}^T) = 0\\
		\label{eq: zx}\underline{R}_{j,i} * (\underline{c}_{z}^{i} {\underline{c}_{x}^{i}}^T + \underline{c}_{x}^{i} {\underline{c}_{z}^{i}}^T) = 0\\
		\label{eq: zy}\underline{R}_{j,i} * (\underline{c}_{z}^{i} {\underline{c}_{y}^{i}}^T + \underline{c}_{y}^{i} {\underline{c}_{z}^{i}}^T) = 0
	\end{align}
\end{subequations}
When the gravity $\boldsymbol{G}\neq \boldsymbol{0}$, there are tow possible cases. If the direction of the gravity coincides with the direction of the rotational axis, i.e. $\boldsymbol{L}_k \times \boldsymbol{G} = \boldsymbol{0}$, the resultant acceleration $\ddot{\boldsymbol{O}}$ is located on the line $\boldsymbol{L}_k$. It means that \eqref{eq: zero accelerate} still holds, and so is \eqref{eq: R fixed}. Otherwise, \eqref{eq: zero accelerate} is broken, causing \eqref{eq: Ox} and \eqref{eq: Oy} not satisfied any more.
\begin{remark}
	In the cases when $\boldsymbol{G} = \boldsymbol{0}$, fixed points existing on the rigid body not directly connecting to the fixed base would cause extra unidentifiable parameters. For example, the fixed points on the second link of Puma560 would cause 2 extra unidentifiable parameters. We leave this case to the future work, and only the cases when $\boldsymbol{G}\neq \boldsymbol{0}$ is considered in this paper.
\end{remark}
\subsection{Principle 3: planar rotations}
With the above principles, the main part of the basis of $Null(Y)$ can be found. Nevertheless, it is common in robotic system when rigid bodies are constrained such that they can only rotate in a plane. For example, the two cases illustrated in Fig.\ref{fig: shared points}c. In this subsection, we explain that this constraint introduces extra basis vectors of of $Null(Y)$.

Because the U, S and F joints all cause spatial rotations, the planar rotations can only generated by a successive R and P joints from the fixed-base to the constrained bodies. It leaves two cases: the body is constrained to its parent through R or P joint. Suppose the direction of the planar rotational axis are recorded and denoted as an infinite point $\boldsymbol{d}_{pr}$. Then, for both cases, this infinite point keeps static, that is
\begin{equation}
	\label{eq: z=0} \ddot{\boldsymbol{d}}_{pr} \equiv \boldsymbol{0}
\end{equation}

The case when the body is connected to its parent through R joint is first considered. We suppose the tetrahedral points of the rigid body constrained to planar rotations are chosen such that $\boldsymbol{O}$ and $\boldsymbol{z}$ are located on the rotational axis, and $\{\boldsymbol{x},\boldsymbol{y}, \boldsymbol{z}\}$ are orthogonal to each other (the left figure of Fig.\ref{fig: shared points}c. Then, there is $\boldsymbol{z}\equiv \boldsymbol{d}_{pr}$. The planar rotations constraint means that the rigid body can only rotate in the direction of $\boldsymbol{z}$, but it may translate in any direction. The acceleration $\ddot{\boldsymbol{O}}$ can be any value. Then, the configuration of $\boldsymbol{x}, \boldsymbol{y}$ can always be formulated as
\begin{align}
	\label{eq: x pos}\boldsymbol{x} &= C_{\theta} \boldsymbol{x}^0 + S_{\theta} \boldsymbol{y}^0\\
	\label{eq: y pos}\boldsymbol{y} &= -S_{\theta} \boldsymbol{x}^0 + C_{\theta} \boldsymbol{y}^0
\end{align}
where $C_{\theta}=\cos{\theta}, S_{\theta} = \sin{\theta}$. $\boldsymbol{x}^0, \boldsymbol{y}^0$ are the initial position of $\boldsymbol{x}, \boldsymbol{y}$ when the rotational angle $\theta = 0$.
Accordingly, the accelerations can be calculated:
\begin{align}
	\label{eq: x acc}\ddot{\boldsymbol{x}} = -\theta^2 \boldsymbol{x} + \ddot{\theta} \boldsymbol{y}\\
	\label{eq: y acc}\ddot{\boldsymbol{y}} = -\ddot{\theta} \boldsymbol{x} - \dot{\theta}^2 \boldsymbol{y}
\end{align}
By simple computation, it is found that
\begin{subequations}
	\label{eq: planar rotation}
	\begin{align}
		\boldsymbol{y} \vee \ddot{\boldsymbol{y}} - \boldsymbol{x} \vee \ddot{\boldsymbol{x}} =& \boldsymbol{0}\\
		\boldsymbol{x} \vee \ddot{\boldsymbol{y}} + \boldsymbol{y} \vee \ddot{\boldsymbol{x}} =& \boldsymbol{0}
	\end{align}    
\end{subequations}

Furthermore, considering \eqref{eq: z=0}, we got
\begin{equation}
	\label{eq: d zero}
	\boldsymbol{L}_j^i \wedge (\boldsymbol{P} \vee \ddot{\boldsymbol{z}}) = \boldsymbol{L}_j^i \wedge (\boldsymbol{P} \vee \ddot{\boldsymbol{d}}_{pr}) =0
\end{equation}
where $\boldsymbol{P}$ can be any point.
Considering the velocity of the rigid body only contains rotational velocity along $\boldsymbol{z}$, we can conclude that
\begin{equation}
	\label{eq: d on the line}
	\boldsymbol{L}_j^i \wedge (\boldsymbol{z}\vee \ddot{\boldsymbol{P}}) = \boldsymbol{L}_j^i \wedge (\boldsymbol{d}_{pr} \vee \ddot{\boldsymbol{P}} )=0
\end{equation}

As a result, the planar rotation introduces four basis of $Null(Y)$:
\begin{subequations}
	\label{eq: coordinates pr}
	\begin{align}
		\label{eq: planar rotation xx yy}\underline{R}_{j,i} * (\underline{c}_{x}^{i} {\underline{c}_{x}^{i}}^T - \underline{c}_{y}^{i} {\underline{c}_{y}^{i}}^T) = 0\\
		\label{eq: planar rotation xy}\underline{R}_{j,i} * (\underline{c}_{x}^{i} {\underline{c}_{y}^{i}}^T + \underline{c}_{y}^{i} {\underline{c}_{x}^{i}}^T) = 0\\
		\label{eq: planar rotation zx} \underline{R}_{j,i} * (\underline{c}_{z}^{i} {\underline{c}_{x}^{i}}^T + \underline{c}_{x}^{i} {\underline{c}_{z}^{i}}^T) = 0\\
		\label{eq: planar rotation zy} \underline{R}_{j,i} * (\underline{c}_{z}^{i} {\underline{c}_{y}^{i}}^T + \underline{c}_{y}^{i} {\underline{c}_{z}^{i}}^T) = 0
	\end{align}
\end{subequations}
It should be noted that \eqref{eq: planar rotation zx} is identical to \eqref{eq: zx}, and \eqref{eq: planar rotation zy} is identical to \eqref{eq: zy}. Therefore, if the rigid body is attached to the fixed base through an R joint, \eqref{eq: planar rotation zx} and \eqref{eq: planar rotation zy} should be omitted. A typical case is illustrated in the left figure of Fig.\ref{fig: shared points}c.

As for the case when the body is connected to its parent through a P joint, the relations in \eqref{eq: coordinates pr} are not independent considering the shared points principle. However, if the parent body is not the fixed base, extra basis vectors are still generated. Suppose $\boldsymbol{O}$ is the Euclidean point in the tetrahedral points. According to \eqref{eq: d zero} and \eqref{eq: d on the line}, the following equation holds
\begin{equation}
	\label{eq: planar rotation Oz} \underline{R}_{j,i} * (\underline{c}_{d_{pr}}^{i} {\underline{c}_{O}^{i}}^T + \underline{c}_{O}^{i} {\underline{c}_{d_{pr}}^{i}}^T) = 0
\end{equation}
A special case is when the translational direction is along the rotational axis, which is $\boldsymbol{z}\equiv \boldsymbol{d}_{pr}$ so that $\boldsymbol{z} \vee \boldsymbol{d}_{pr}=\boldsymbol{0}$. It is illustrated in the right figure of Fig.\ref{fig: shared points}c. The third body is connected to the second body, which is constrained to planar rotations, through a P joint. The translational direction coincides with the planar rotational axis. In this case, the $\boldsymbol{L}_k$ in \eqref{eq: jacobian line} is an infinite line, which means 
\begin{align}
	\boldsymbol{L}_k \wedge (\boldsymbol{x} \vee \ddot{\boldsymbol{O}}) &= 0\\
	\boldsymbol{L}_k \wedge (\boldsymbol{y} \vee \ddot{\boldsymbol{O}}) &= 0
\end{align}
Besides, $\boldsymbol{L}_k$ is along the direction of $\boldsymbol{z}$. Considering \eqref{eq: x acc} and \eqref{eq: y acc}, we got the following equations:
\begin{align}
	\boldsymbol{L}_k \wedge (\boldsymbol{O} \vee \ddot{\boldsymbol{x}}) &= 0\\
	\boldsymbol{L}_k \wedge (\boldsymbol{O} \vee \ddot{\boldsymbol{y}}) &= 0
\end{align}
It results in the transfer of the inertia to the parent body as formulated in \eqref{eq: coeff dependency}. Therefore, two new linear dependencies are generated in the similar form of shared points:
\begin{equation}
	\begin{split}
		\label{eq: Ox planar rotations}
		&\underline{R}_{j,i} * (\underline{c}_{O}^{i} {\underline{c}_{x}^{i}}^T + \underline{c}_{x}^{i} {\underline{c}_{O}^{i}}^T) \\
		- &\underline{R}_{j,\lambda(i)} * (\underline{c}_{O}^{\lambda(i)} {\underline{c}_{x}^{\lambda(i)}}^T + \underline{c}_{x}^{\lambda(i)} {\underline{c}_{O}^{\lambda(i)}}^T) = 0
	\end{split}
\end{equation}
\begin{equation}
	\begin{split}
		\label{eq: Oy planar rotations}
		&\underline{R}_{j,i} * (\underline{c}_{O}^{i} {\underline{c}_{y}^{i}}^T + \underline{c}_{y}^{i} {\underline{c}_{O}^{i}}^T) \\
		- &\underline{R}_{j,\lambda(i)} * (\underline{c}_{O}^{\lambda(i)} {\underline{c}_{y}^{\lambda(i)}}^T + \underline{c}_{y}^{\lambda(i)} {\underline{c}_{O}^{\lambda(i)}}^T) = 0
	\end{split}
\end{equation}

With the above three principles, all the basis vectors of $Null(Y)$ can be determined, and the basis of base parameter space $B(Y)$ can be further generated according to \eqref{eq: base space}.
\section{Algorithm implementations}
\label{sec: Algorithm}
In this section, we introduce an iterative algorithm to analyze the base parameters of a given robotic system. The complexity of the algorithm is at worst $O(N_B+N_L)$ and at best $O(1)$. The main problem is to figure out the coordinates of joint points in body-fixed frames. Therefore, we introduce the determination of joint points and the assignment of the body tetrahedral points, successively.
\subsection{Representation of robots}
\begin{figure}[t]
	\centering
	\includegraphics[width=0.9\linewidth]{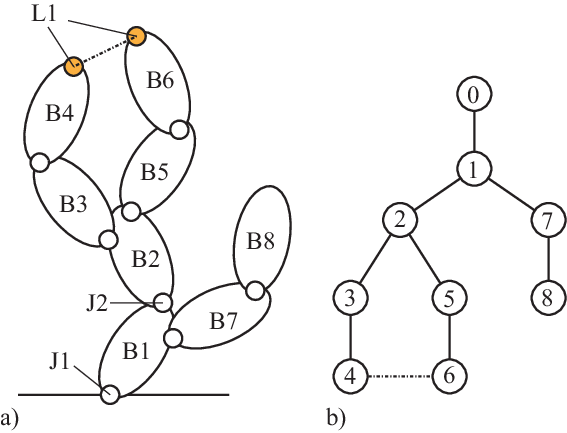}
	\caption{Illustration of the URDF+ convention (\cite{Chignoli2024urdfenhancedurdfrobots}). a). The illustration of a tree-structure robot with a closed loop. The closed loop is cut off and a set of constraints is introduced. b). The graphic topology of the robot. Each vertex is a rigid body, and each edge is a joint. The loop joints are illustrated with dashed edges. The indices are assigned according to the depth-first search convention.}
	\label{fig: robot topology}
\end{figure}
We apply the convention of the URDF+ (\cite{Chignoli2024urdfenhancedurdfrobots}) to represent a robot. For robots without closed loops, such as serial robot, they are represented as kinematic trees. Every rigid body in the robot has only one parent. Its body-fixed frame is defined in the joint connecting it to its parent. Therefore, the joint points shared by the parent and the child always lie in the origin or the x/y/z axis of the frame. The convention of the shared points settings are summarized in Table. \ref{tab: shared points}. The number of joints is identical to the number of rigid bodies $N_B$, and these joints are called tree joints. The configuration of the robot can be determined by tree joint variables, denoted as $\underline{q}$.  

However, if the robot contains closed loops, one joint in each loop should be cut to transform the topology of the robot to a tree structure. The tree joint variables are still denoted as $\underline{q}$. The cut joints are denoted as loop joints. Additional frames are required to be assigned at loop joints, and they should follow the same convention in Table \ref{tab: shared points}. The kinematics constraints are constructed in the form of algebraic equations:
\begin{equation}
	C(\underline{q})=\underline{0}
\end{equation}
where the coordinates $\underline{q}$ contain the actuated coordinates $\underline{q}_a$ and the passive coordinates $\underline{q}_p$. The velocity-level constraints are
\begin{equation}
	\label{eq: jacobian}
	\underline{J}_a \underline{q}_a + \underline{J}_p \underline{q}_p = \underline{0}
\end{equation}
The rank-deficiency of $\underline{J}_p$ can cause singularity, and its condition number can provide a measure of the ``distance" to singularities.

Consequently, the robot with closed loops is depicted as a constrained tree-structured robot. In this case, the number of joints is $N_B+N_L$, where $N_L$ is the number of loop joints. The robot's configuration is determined by the independent joints, which is usually the actuated tree joints in practice. We denote it as $\underline{q}_a$, and the rest of tree joint variables as $\underline{q}_p$. In the case of tree-structured robots, $\underline{q}_a=\underline{q}$. An example is illustrated in Fig.\ref{fig: robot topology}a.

The indices of the rigid bodies are set from the fixed base to the leaves of the kinematic tree according to the depth-first search convention. The indices of the joints in the kinematic tree are identical to the child bodies. The indices of the loop joints are set by users. It should be noted that the parent index of a loop joint should be smaller than the child index. An example is illustrated in Fig.\ref{fig: robot topology}b.

With the above convention, all the frames are assigned on joints. They are called joint frames. With proper kinematics algorithms, one can calculate the initial configuration of all joint frames when $\underline{q}_a = \underline{q}_0$. In PGA, we use motors to represent the rigid body configuration instead of homogeneous matrices. In the implementations, a robot is represented as a struct variable ``rbt" containing the attributes summarized in Table \ref{tab: robot representation}. It should be noted that only the attributes involved in base parameter analysis are included in this table.
\begin{table}[h]
	\centering
	\begin{tabular}{ccc}
		\toprule
		Name  & Meaning & Size \\ \noalign{\smallskip}
		\midrule
		n  & Number of rigid bodies & 1\\ \noalign{\smallskip}
		type  & Types of tree joints & [n,1] \\ \noalign{\smallskip}
		parent  & Indices of parent body & [n,1]\\ \noalign{\smallskip}
		nl  & Number of loops & 1\\ \noalign{\smallskip}
		ltype  & Types of loop joints & [nl,1]\\ \noalign{\smallskip}
		lparent  & Ids of parents of loop joints & [nl,1]\\ \noalign{\smallskip}
		lchild  & Ids of children of loop joints & [nl,1]\\ \noalign{\smallskip}
		q\_0  & Initial configuration $\underline{q}_0$  & [dof\_a, 1]\\ \noalign{\smallskip}
		M0  & Initial config of rigid bodies  & [8,n]\\ \noalign{\smallskip}
		Ml  & Initial config of loop joints  & [8,nl]\\ \noalign{\smallskip}
		L0  & 1-DOF joints' axis & [6,n]\\ \noalign{\smallskip}
		is\_pr  & Indicators for planar rotations  & [n,1]\\ \noalign{\smallskip}
		d\_pr  & Directions of planar rotational axis  & [3,n]\\ \noalign{\smallskip}
		g  & the gravity  & [6,1]\\ \noalign{\smallskip}
		\bottomrule
	\end{tabular}
	\caption{The attributes in the struct variable rbt.}
	\label{tab: robot representation}
\end{table}
\subsection{Spatial assignment of tetrahedral points}
\begin{figure}[t]
	\centering
	\includegraphics[width=0.9\linewidth]{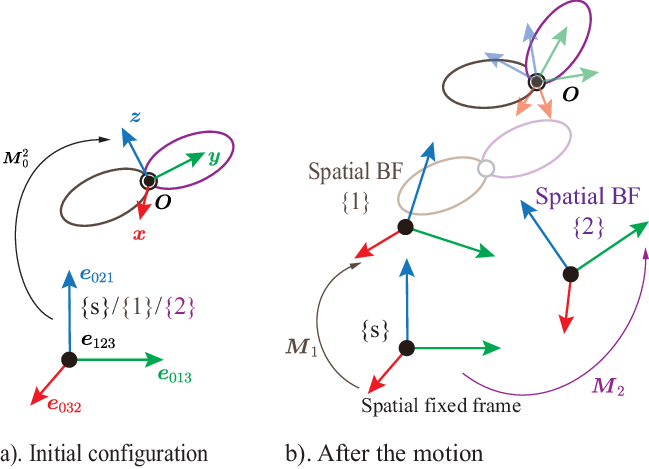}
	\caption{Illustration of the spatial assignment of tetrahedral points. a). The initial configuration when the tetrahedral points of body 1 and body 2 are coincide with the spatial fixed frame. b). After the rigid body motion, the spatial body-fixed frames (BFs) of the two rigid bodies are separated, but the coordinates of point $\boldsymbol{O}$ keeps the same in the two BFs.}
	\label{fig: spatial assignment}
\end{figure}
Each tree joint frame introduced above is attached to a rigid body, and is a direct choice as the tetrahedral points for dynamics modeling. However, with this assignment, extra attributes have to be included in the robot model to record the coordinates of shared points w.r.t. the parent body and the child body. For example, in the case of R joint, shared points include $\boldsymbol{O}$ and $\boldsymbol{z}$ of the joint frame, and the coordinates $\underline{c}_O^i, \underline{c}_O^{\lambda(i)}, \underline{c}_z^i, \underline{c}_z^{\lambda(i)},$ should be all calculated and record in the robot model. It takes more memory to save data and extra effort to calculate all the coordinates.

In order to address this issue, we apply the ``spatial" assignment of tetrahedral points in our algorithm. First, all the rigid bodies are set to their initial configurations represented as $\boldsymbol{M}_0^i$. $\boldsymbol{M}_0^i$ is the rigid body motion from the spatial fixed frame $\{s\}$ to the body fixed frame $\{i\}$ at the initial configuration. Then, the frame coincides with the spatial fixed frame at initial configuration, instead of the joint frame, is assigned as the body-fixed frame of each rigid body. For example, in Fig.\ref{fig: spatial assignment}a, rigid body 1 and 2 are connected through an S joint. The joint frame is illustrated as $\{\boldsymbol{O}, \boldsymbol{x}, \boldsymbol{y}, \boldsymbol{z}\}$, and the spatial fixed frame is illustrated as $\{\boldsymbol{e}_{123}, \boldsymbol{e}_{032}, \boldsymbol{e}_{013}, \boldsymbol{e}_{021}\}$. At the initial configuration, the body fixed frames $\{1\}$ and $\{2\}$ according to the spatial assignment are both coincide with $\{s\}$. In a general case, the body-fixed frames of all rigid bodies coincide at the initial configuration, and the positions of the corresponding tetrahedral points are the same: 
\begin{subequations}
	\begin{align}
		\boldsymbol{x}_{s0}^i = \boldsymbol{e}_{032}\\
		\boldsymbol{y}_{s0}^i = \boldsymbol{e}_{013}\\
		\boldsymbol{z}_{s0}^i = \boldsymbol{e}_{021}\\
		\boldsymbol{O}_{s0}^i = \boldsymbol{e}_{123}
	\end{align}
\end{subequations}
Suppose the tetrahedral points corresponding to the joint frame are denoted as $\{\boldsymbol{x}^i_{b}, \boldsymbol{y}^i_{b}, \boldsymbol{z}^i_{b}, \boldsymbol{O}^i_{b}\}$. They are related to the ones under spatial assignment through the rigid body motion $\boldsymbol{M}^i_0$:
\begin{subequations}
	\label{eq: coordinates shared points}
	\begin{align}
		\boldsymbol{x}_{b0}^i = \widetilde{\boldsymbol{M}}_0^i \boldsymbol{x}_{s0}^i \boldsymbol{M}_0^i = \underline{\boldsymbol{E}}^i_{s} \underline{c}_{x_i}^i\\
		\boldsymbol{y}_{b0}^i = \widetilde{\boldsymbol{M}}_0^i \boldsymbol{y}_{s0}^i \boldsymbol{M}_0^i = \underline{\boldsymbol{E}}^i_{s} \underline{c}_{y_i}^i\\
		\label{eq: czi}\boldsymbol{z}_{b0}^i = \widetilde{\boldsymbol{M}}_0^i \boldsymbol{z}_{s0}^i \boldsymbol{M}_0^i = \underline{\boldsymbol{E}}^i_{s} \underline{c}_{z_i}^i\\
		\label{eq: cOi}\boldsymbol{O}_{b0}^i = \widetilde{\boldsymbol{M}}_0^i \boldsymbol{O}_{s0}^i \boldsymbol{M}_0^i = \underline{\boldsymbol{E}}^i_{s} \underline{c}_{O_i}^i
	\end{align}
\end{subequations}
where $\underline{\boldsymbol{E}}^i_{s}=[\boldsymbol{x}_{s0}^i, \boldsymbol{y}_{s0}^i, \boldsymbol{z}_{s0}^i, \boldsymbol{O}_{s0}^i]$ is a row array, and $\underline{c}_{\square}^i$ is the coordinates w.r.t. $\underline{\boldsymbol{E}}^i_{s}$. As illustrated in Fig.\ref{fig: spatial assignment}a, the points $\{\boldsymbol{O}, \boldsymbol{x}, \boldsymbol{y}, \boldsymbol{z}\}$ can be obtained by moving $\{\boldsymbol{e}_{123}, \boldsymbol{e}_{032}, \boldsymbol{e}_{013}, \boldsymbol{e}_{021}\}$ through $\boldsymbol{M}_0^2$. During the rigid body motion, the tetrahedral points $\underline{\boldsymbol{E}}^i_{s}$ are moving while the coordinates $\underline{c}_{\square}^i$ keep static. Because all the $\underline{\boldsymbol{E}}^i_{s}$ coincide with each other at the initial configuration, the shared points have the same coordinates w.r.t. both the parent and the child body. For example, in Fig.\ref{fig: spatial assignment}b, $\boldsymbol{O}$ are shared points, and its coordinates w.r.t. the parent and the child body are the same:
\begin{equation}
	\underline{c}_{O_i}^i=\underline{c}_{O_i}^{\lambda(i)}
\end{equation}
and they are calculated according to \eqref{eq: cOi}. With the same logic, it can be proved that coordinates of shared points on loop joints can be calculated with $\boldsymbol{M}_l$, and they are the same w.r.t. the parent frame and the child frame. Therefore, no extra attributes are required to be included in the Table \ref{tab: robot representation}, and coordinates of all shared points can be calculated according to \eqref{eq: coordinates shared points}.

Then, we propose the spatial shared points (SSP) algorithm to calculate the coordinates of shared points in parent and child frames as summarized in Algorithm \ref{alg: SSP}. In the algorithm, the function $\operatorname{Coord}()$ returns the coordinates of points as homogeneous coordinates. The floating base is considered connected with the fixed base through a floating (F) joint.
\begin{algorithm}
	\caption{The algorithm SSP()}
	\label{alg: SSP}
	\begin{algorithmic}[1]
		\REQUIRE The initial configuration $\boldsymbol{M}_0$, the joint type $t$
		\STATE $[\underline{x},\underline{y},\underline{z},\underline{O}]=\underline{I}_4$
		\IF{$t$=`R'}
		\STATE $\underline{c}^p=\underline{c}^c=\operatorname{Coord}\left(\widetilde{\boldsymbol{M}}_0[\boldsymbol{z}, \boldsymbol{O}]\boldsymbol{M}_0\right)$ 
		\ELSIF{$t$=`P'}
		\STATE $\underline{c}^p=\underline{c}^c=\operatorname{Coord}\left(\widetilde{\boldsymbol{M}}_0[\boldsymbol{x}, \boldsymbol{y}, \boldsymbol{z}]\boldsymbol{M}_0\right)$ 
		\ELSIF{$t$=`U' \OR `S'}
		\STATE $\underline{c}^p=\underline{c}^c=\operatorname{Coord}\left(\widetilde{\boldsymbol{M}}_0\boldsymbol{O}\boldsymbol{M}_0\right)$ 
		\ELSIF{$t$=`F'}
		\STATE $\underline{c}^p=\underline{c}^c=[\ ]$ 
		\ENDIF
		\RETURN $\underline{c}^p$, $\underline{c}^c$
	\end{algorithmic}
\end{algorithm}

\subsection{Base parameter analysis}
Based on the robot model and the spatial assignment, we propose the algorithm to analyze base parameters of a robot's inertia. Because the three principles are applied to each joint including tree joints and loop joints once the model rbt in Table \ref{tab: robot representation} is obtained, it results in the $O(N_B+N_L)$ complexity with serial iteration and $O(1)$ complexity with parallel iteration. 

Before formally conduct the obervability analysis, the rigid bodies following planar rotations should be identified. We achieve this goal with the following algorithm, which is called the planar rotations indicator (PRI) algorithm. It is an $O(N_B)$-complexity algorithm. In the algorithm, ${\boldsymbol{L}_0}_{[1:3]}$ takes the first three coordinates of the initial lines of 1-DOF joints, which is potential to create planar rotations. For R joints, it indicates the direction of the axis ,and for the other joints it is zero vector. The criteria $\boldsymbol{L}_{d_i} \times \boldsymbol{L}_{d_{\lambda(i)}}=\boldsymbol{0}$ is to determine if the rotational axis is parallel to the one of the parent body. In the cases of P joints, this criteria is practical because $\boldsymbol{L}_{d_i}$ is zero and no rotations are involved. We only consider the cases when planar rotations are caused by tree joints, and leave it in the modeling stage to avoid loop joints causing planar rotations. The output of the algorithm is added as attributes of rbt.
\begin{algorithm}
	\caption{The algorithm PRI()}
	\label{alg:PRI}
	\begin{algorithmic}[1]
		\REQUIRE The number of rigid bodies $n$, indices of parent bodies $\lambda$, types of tree joints $t$, directions of 1-DOF axis $\boldsymbol{L}_d = {\boldsymbol{L}_0}_{[1:3]}$
		\STATE Initialize $is_{pr}$ as a $n\times 1$ zeros array
		\STATE Initialize $d_{pr}$ as a $3\times n$ zeros array
		\FOR{$i=1:n$}
		\IF{$t(i)$=`R' \OR `P'}
		\IF{$\lambda(i)$=0}
		\STATE ${is_{pr}}(i)=1$
		\STATE $d_{pr}(:,i) = {\boldsymbol{L}_d}_i$
		\ELSIF{$is_{pr}(\lambda(i))$=1 \AND $\boldsymbol{L}_{d_i} \times \boldsymbol{L}_{d_{\lambda(i)}}=\boldsymbol{0}$}
		\STATE ${is_{pr}}(i)=1$
		\IF{$\operatorname{norm}(\boldsymbol{L}_{d_{\lambda(i)}})\neq 0$}
		\STATE $d_{pr}(:,i) = \boldsymbol{L}_{d_{\lambda(i)}}$
		\ELSIF{$\operatorname{norm}(d_{pr}(:,p(i)))\neq 0$}
		\STATE $d_{pr}(:,i) = d_{pr}(:,p(i)) $
		\ELSE
		\STATE $d_{pr}(:,i) = \boldsymbol{L}_{d_i} $
		\ENDIF
		\ENDIF
		\ENDIF
		\ENDFOR
		\RETURN $is_{pr}$, $d_{pr}$
	\end{algorithmic}
\end{algorithm}

Furthermore, we noted that 
\begin{equation}
	SDP(4)\in \mathbb{S}(4) \cong \mathbb{R}^{10}
\end{equation}
where $\mathbb{S}(4)$ is the space of 4th-order symmetric matrices. Therefore, we use vectors in $\mathbb{R}^{10}$ to represent inertial parameters, and the following operator is defined to map elements in $\mathbb{S}(4)$ to $\mathbb{R}^{10}$.
\begin{equation}
	\begin{split}
		\widehat{\underline{N}} = &\left[N_{11}, N_{12}, N_{13}, N_{22}, N_{23}, N_{33}, \right.\\
		&\left. N_{14}, N_{24}, N_{34}, N_{44}\right]^T \in \mathbb{R}^{10}
	\end{split}
\end{equation} 
With this operator, we propose the Joint Nullspace Generator (JNG) algorithm to implement Principle 1. Then, the Principle 2 and 3 are implemented according to the equations \eqref{eq: U/S fixed}, \eqref{eq: P fixed}, \eqref{eq: R fixed}, and \eqref{eq: planar rotation}. The pseudo codes of the algorithm, named as the Dynamics Regressor Nullspace Generator (DRNG), are summarized in Algorithm \ref{alg: DRNG}. It is left to the modeling procedure to make sure fixed points and planar rotations only caused by tree joints. From the pseudo codes, it is found that all the iterations are independent. Therefore, it is possible to achieve $O(1)$ complexity with good parallelization. The nullspace is $Null(Y) = \operatorname{span}(B_{null})$, and the base parameter space $B(Y)$ is then generated by \eqref{eq: base space}.
\begin{algorithm}
	\caption{The algorithm JNG()}
	\label{alg: JNSG}
	\begin{algorithmic}[1]
		\REQUIRE The number of rigid bodies $n$, The number of shared points $n_s$, the parent's index $p$, the child's index $c$, coordinates w.r.t the parent frame $\underline{c}^{p}$, coordinates w.r.t the child frame $\underline{c}^{c}$
		\STATE Initialize $j_{null}$ as a $10n \times \frac{n_s (n_s+1)}{2}$ zeros array
		\FOR{$i=1:n_s$}
		\FOR{$j=i:n_s$}
		\STATE $k = \frac{(2n_s-i+2)(i-1)}{2}+j-i+1$
		\IF{$j=i$}
		\IF{$p>0$}
		\STATE $j_{null}(10p-9:10p,k)=\widehat{\underline{c}^p_{j}{\underline{c}^p_{j}}^T}$
		\ENDIF
		\STATE $j_{null}(10c-9:10c,k)=\widehat{\underline{c}^c_{j}{\underline{c}^c_{j}}^T}$
		\ELSE
		\IF{$p>0$}
		\STATE $j_{null}(10p-9:10p,k)=\widehat{\left(\underline{c}^p_{i}{\underline{c}^p_{j}}^T + \underline{c}^p_{j}{\underline{c}^p_{i}}^T\right)}$
		\ENDIF
		\STATE $j_{null}(10c-9:10c,k)=\widehat{\left(\underline{c}^c_{i}{\underline{c}^c_{j}}^T + \underline{c}^c_{j}{\underline{c}^c_{i}}^T\right)}$
		\ENDIF
		\ENDFOR
		\ENDFOR
		\RETURN $j_{null}$
	\end{algorithmic}
\end{algorithm}

\begin{algorithm}
	\caption{The algorithm DRNG()}
	\label{alg: DRNG}
	\begin{algorithmic}[1]
		\REQUIRE The struct variable rbt defined in Table. \ref{tab: robot representation}
		\STATE Parse the variables from rbt: $n$=rbt.n, $nl$ = rbt.nl, $is_{pr}$ = rbt.is\_pr, $\boldsymbol{G}$ = rbt.g, $t$=rbt.type, $t_l$=rbt.lt,  $\lambda$=rbt.parent, $\lambda_l$=rbt.lparent, $\sigma_l$=rbt.lchild, $\boldsymbol{M}_0$=rbt.M0, $\boldsymbol{M}_{l0}$=rbt.Ml
		\STATE $B_{null}=\operatorname{cell}(n+n_l,1)$
		\FOR{$i=1:n$}
		\STATE $b_{null}=[\ ]$
		\STATE \textbf{Principle 1:}
		\STATE $\underline{c}^p, \underline{c}^c=\operatorname{SSP}(\boldsymbol{M}_{0}(i), t(i))$
		\STATE $p=\lambda(i), c=i, n_s=\operatorname{size}(\underline{c}^p,2)$
		\STATE $j_{null}=\operatorname{JNG}(n,n_s,p,c,\underline{c}^p, \underline{c}^c)$
		\STATE $b_{null}=[b_{null}, j_{null}]$
		\IF{$p$=0}
		\STATE \textbf{Principle 2:}
		\IF{$t(i)$=`S' \OR `U'}
		\IF{$\boldsymbol{G}\neq \boldsymbol{O}$}
		\STATE $j_{null}\leftarrow$\eqref{eq: U/S fixed}
		\ENDIF
		\ELSIF{$t(i)$=`P'}
		\STATE $j_{null}\leftarrow$\eqref{eq: P fixed}
		\ELSIF{$t(i)$=`R'}
		\IF{$\boldsymbol{L}_0(i) \times \boldsymbol{G} \neq \boldsymbol{O}$}
		\STATE $j_{null}\leftarrow$\eqref{eq: zx}\eqref{eq: zy}
		\ELSE 
		\STATE $j_{null}\leftarrow$\eqref{eq: R fixed}
		\ENDIF
		\ENDIF
		\STATE $b_{null}=[b_{null}, j_{null}]$
		\ENDIF
		\IF{$is_{pr}(i)$=1}
		\STATE \textbf{Principle 3:}
		\IF{$t(i)$=`R'}
			\IF{$\lambda(i)$=0}
			\STATE $j_{null}\leftarrow$\eqref{eq: planar rotation xx yy} \eqref{eq: planar rotation xy}
			\ELSE 
			\STATE $j_{null}\leftarrow$\eqref{eq: coordinates pr}
			\ENDIF
		\ELSIF{$t(i)$=`P' and $p(i)\neq$0}
			\IF{$\boldsymbol{d}_{pr} \vee \boldsymbol{z} = \boldsymbol{0}$}
				\STATE $j_{null}\leftarrow$\eqref{eq: planar rotation Oz} \eqref{eq: Ox planar rotations} \eqref{eq: Oy planar rotations}
			\ELSE
				\STATE $j_{null}\leftarrow$\eqref{eq: planar rotation Oz}
			\ENDIF
		\ENDIF
		\STATE $b_{null}=[b_{null}, j_{null}]$
		\ENDIF
		\STATE $B_{null}\{i\}=b_{null}$
		\ENDFOR
		\FORALL{$l=1:n_l$}
		\STATE $\underline{c}^p, \underline{c}^c=\operatorname{SSP}(\boldsymbol{M}_{l0}(i), t_l(i))$
		\STATE $p=\lambda_l(i), c=\sigma_l(i), n_s=\operatorname{size}(\underline{c}^p,2)$
		\STATE $j_{null}=\operatorname{JNG}(n,n_s,p,c,\underline{c}^p, \underline{c}^c)$
		\STATE $B_{null}\{i\}=j_{null}$
		\ENDFOR
		\RETURN $B_{null}$
	\end{algorithmic}
\end{algorithm}

\subsection{Regressor computation}
In order to validate the DRNG algorithm, it is necessary to calculate the dynamics regressor. We suppose the configuration of each rigid body $\boldsymbol{M}$, its spatial velocity $\boldsymbol{V}$, and its spatial acceleration $\dot{\boldsymbol{V}}$ are calculated using kinematics algorithms. Then, the dynamics regressor is first computed in tensor form $\underline{R}$ and rearranged into matrix form $\underline{Y}$. The Dynamics Regressor (DR) algorithm is summarized in Algorithm \ref{alg: DR}. In the algorithm, the Jaccobian $\underline{\boldsymbol{J}}_{V}(i) = [\boldsymbol{L}^i_1,...,\boldsymbol{L}^i_{n_a}]$ is an array of lines such that
\begin{equation}
	\boldsymbol{V}(i) = \sum_{k=1}^{n_a}\boldsymbol{L}^i_{k}\dot{q}_{a_k}
\end{equation}
Besides, the dot wedge $.\wedge$ is an elementary operator such that
\begin{equation}
	[\boldsymbol{L}_1,...,\boldsymbol{L}_N] . \wedge \boldsymbol{L} = [\boldsymbol{L}_1\wedge \boldsymbol{L},...,\boldsymbol{L}_N\wedge \boldsymbol{L}]
\end{equation}
It returns an $N$-dimensional real number valued array. In order to transform the tensor form dynamics regressor to the matrix form, the bar operator is defined such that
\begin{equation}
	\begin{split}
		\bar{\underline{R}} = [&R_{11};R_{12}+R_{21};R_{13}+R_{31};R_{22};\\
		&R_{23}+R_{32};R_{33};R_{14}+R_{41};\\
		&R_{24}+R_{42};R_{34}+R_{43};R_{44}]
	\end{split}
\end{equation}

\begin{algorithm}
	\caption{The algorithm DR()}
	\label{alg: DR}
	\begin{algorithmic}[1]
		\REQUIRE The configuration $\boldsymbol{M}$, velocity $\boldsymbol{V}$, acceleration $\dot{\boldsymbol{V}}$, and the Jaccobian $\underline{\boldsymbol{J}}_{V}$ w.r.t. $\underline{q}_a$ of each body, the gravity $\boldsymbol{G}$, the DOF $n_a$, the number of rigid bodies $n$.
		\STATE $\underline{\boldsymbol{E}} = [\underline{x},\underline{y},\underline{z},\underline{O}]=\underline{I}_4$
		\STATE $\underline{R} = \underline{0}_{n_a\times n \times 4 \times 4}$
		\STATE $\underline{Y} = \underline{0}_{n_a\times 10n}$
		\FOR{$i=1:n$}
		\FOR{$m=1:4$}
		\STATE $\underline{\boldsymbol{E}}^i(m) = \widetilde{\boldsymbol{M}}(i) \underline{\boldsymbol{E}}(m) \boldsymbol{M}$
		\STATE $\underline{\dot{\boldsymbol{E}}}^i(m) = \underline{\boldsymbol{E}}^i(m) \times \boldsymbol{V}(i)$
		\STATE $\underline{\ddot{\boldsymbol{E}}}^i(m) = \underline{\boldsymbol{E}}^i(m) \times (\dot{\boldsymbol{V}}(i) + \boldsymbol{G}) + \underline{\dot{\boldsymbol{E}}}^i(m) \times \boldsymbol{V}(i)$
		\ENDFOR
		\FOR{$m=1:4$}
		\FOR{$N=1:4$}
		\STATE $\underline{R}(:,i,m,n) = \underline{\boldsymbol{J}}_{V}(i) .\wedge \left(\underline{\boldsymbol{E}}^i(m) \vee \underline{\ddot{\boldsymbol{E}}}^i(n)\right)$
		\ENDFOR
		\ENDFOR
		\STATE $\underline{Y}(:,10i-9:10i) = \bar{\underline{R}}(:,i,:,:)$
		\ENDFOR		
		\RETURN $\underline{Y}$
	\end{algorithmic}
\end{algorithm}

\section{Demonstrations}

\begin{table*}[h]
	\centering
	\begin{tabular}{cccc}
		\toprule
		Robot  & Numerical (\cite{Gautier1991NumericalCal}) & RPNA (\cite{Wensing2024Geometric}) & DRNG (proposed) \\ \noalign{\smallskip}
		\midrule
		Puma560  & 24 (74.33 ms) & 24 (54.47 ms) & 24 (0.67 ms)\\ \noalign{\smallskip}
		Unitree Go2  & 36 (144.10 ms) & 36 (59.96 ms)& 36 (1.52 ms) \\ \noalign{\smallskip}
		PKM: 2RRU-1RRS  & 43 (46.36 s, {\color{red} wrong results})& --- & 45 (1.81 ms)\\ \noalign{\smallskip}
		PKM: 2PRS-1PSR  & 47 (41.79 s) & --- & 47 (2.26 ms) \\ \noalign{\smallskip}
		\bottomrule
	\end{tabular}
	\caption{Dimensions of the four robots' dynamics regressor nullspace and run-time of the each algorithm.}
	\label{tab: results}
\end{table*}
\label{sec: demonstrations}
In this section, four examples are provided to demonstrate the proposed DRNG algorithm, and the results are listed in Table. \ref{tab: results}. Based on these examples, it is concluded that the DRNG algorithm is effective and efficient no matter the robot is fixed-base or floating-base, and no matter with or without closed loops. 

The algorithms in this paper are implemented in MATLAB R2024a on a laptop (i5-12500H @ 2.50 GHz). We apply the numerical method (\cite{Gautier1991NumericalCal}) to validate the results of DRNG algorithm. With random samples of $\{\underline{q}_a, \dot{\underline{q}}_a, \ddot{\underline{q}}_a\}$ and kinematics algorithms,  $\boldsymbol{M},\boldsymbol{V},\dot{\boldsymbol{V}}$ of each rigid body in robot can be calculated. Then, with the Algorithm \ref{alg: DR}, the dynamics regressors of each sample can be calculated and concatenated as a matrix $\underline{\mathbb{Y}}$. The output of DRNG is transformed to a matrix $\underline{B}_{null}$, and its complementary space $\underline{B}=\underline{B}_{null}^\perp$ is calculated. Finally, the if the following criteria is satisfied, the output DRNG includes all the basis vectors of $Null(Y)$.
\begin{subequations}
	\label{eq: criteria}
	\begin{align}
		\operatorname{rank(\underline{\mathbb{Y}}\underline{B})} &= \operatorname{rank(\underline{B})}\\
		\|\underline{\mathbb{Y}}\underline{B}_{null}\|_{F}&=0
	\end{align}
\end{subequations} 
where $\|\cdot\|_{F}$ is the Frobenius norm of a matrix.
\subsection{Puma560}
\begin{figure}[t]
	\centering
	\includegraphics[width=0.9\linewidth]{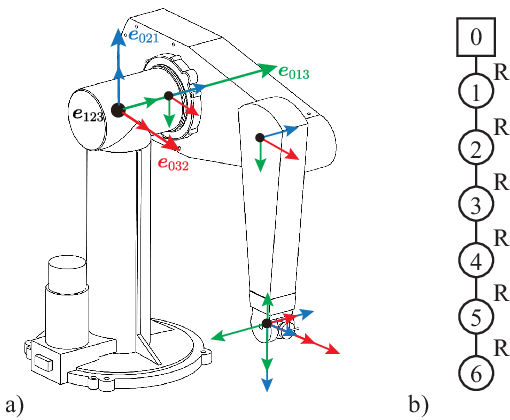}
	\caption{Illustration of the initial configuration and joint frames of Puma560 in a) and the graphic topology in b).}
	\label{fig: Puma560}
\end{figure}
Puma560 is one of the most classical serial robots, and we chose it as a representative for serial robots. It is composed of six rigid bodies connected by R joints, resulting in 6 DOFs. Its initial configuration is illustrated in Fig. \ref{fig: Puma560}. 

We apply the three principles to analyze the base parameters of Puma560's inertia. First, the shared points of the six R joints result in $3\times6=18$ nullspace basis vectors. Second, because the first link is connected to the fixed base through an R joint, it has fixed points and its motion is planar rotations. According to the principle 2, the fixed points of link $1$ result in extra $4$ nullspace basis vectors if the gravity is along z-axis, and $2$ nullspace if the gravity is not along z-axis. According to the principle 3, the planar rotations of link $1$ result in $2$ nullspace basis vectors. Therefore, only $36$ inertial parameters wit z-direction gravity can be identified independently, and $38$ inertial parameters otherwise. This result is consistent with that in \cite{KhalilModeling2002} and the example in \cite{Wensing2024Geometric} when the rotor parameters are not considered.

Then, the Puma560 is represented as a struct variable introduced in Table \ref{tab: robot representation}. The DRNG algorithm takes it as input and returns $\underline{B}_{null}$ and $\underline{B}$.  In numerical validation, 100 samples of $\underline{q}_a, \dot{\underline{q}}_a, \ddot{\underline{q}}_a$ are uniformly generated on the interval $(0,1)$. The concatenated matrix $\underline{\mathbb{Y}}$ has 600 rows and 60 columns. The QR decomposition method is applied to $\underline{\mathbb{Y}}$ for base parameters analysis. The output of DRNG algorithm satisfies the criteria \eqref{eq: criteria}, and it is the same as that of QR decomposition method in both cases when the gravity is along z-axis and along x-axis. 

The run-time of the algorithms is also evaluated on MATLAB 2024a. The numerical method costs $74.33$ms, and the RPNA algorithm proposed in \cite{Wensing2024Geometric} costs $54.47$ms. The proposed DRNG algorithm costs $0.67$ms, which is about $70\times$ faster than the advanced PRNA algorithm. It is mainly because the proposed algorithms only involves simple analytical calculations, without calling the QR decomposition function.

\subsection{Unitree Go2}
\begin{figure}[t]
	\centering
	\includegraphics[width=0.9\linewidth]{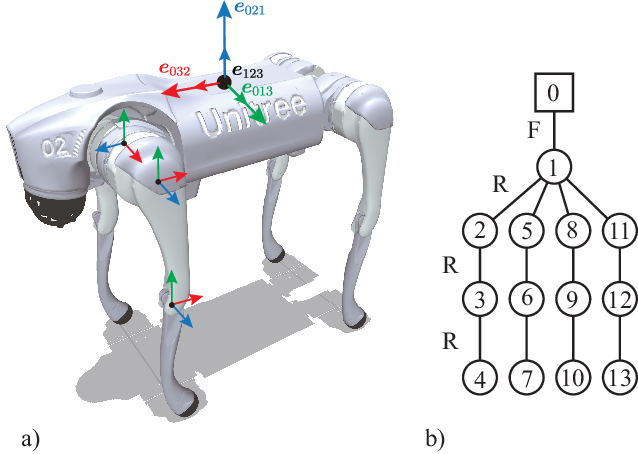}
	\caption{a). Illustration of the initial configuration of Unitree Go2. The joint frames of one leg and the floating joint frame are illustrated. b). The graphic topology.}
	\label{fig: Unitree}
\end{figure}
Unitree Go2 is a widely applied quadruped. It is composed of a floating base and four legs. Each leg can be modeled as three links connected by three R joints. This robot is used as a representative for robots with floating base. Its initial configuration, and tetrahedral points on one leg is illustrate in Fig. \ref{fig: Unitree}.

Because all the rigid bodies in floating-base system undergoes spatial motion, the fixed points and planar rotations do not exist in this kind of robot. Only principle 1 is active for base parameter analysis of inertial parameters. As for the Unitree Go2, it has $13$ rigid bodies and $12$ R joints. Each R joint has two shared points and results in $3$ nullspace basis vectors. Therefore, the nullspace of dynamics regressor has $36$ basis vectors, and $130-36=94$ inertial parameters can be independently identified. 

It should be noted that in the case of floating-base robot, the configuration of floating base is depicted by the Lie group of rigid body motion. In our algorithm, it is the motor group $\boldsymbol{\mathcal{M}}_{3,0,1}$. With the help of the exponential map, we can use elements in Lie algebra $\boldsymbol{m}_{3,0,1}$ to generate elements on $\boldsymbol{\mathcal{M}}_{3,0,1}$. Suppose $\underline{q}_{[1:6]}$ is the configuration coordinates of the floating base, and $\underline{\tau}_{[1:6]}$ is the wrench acting on the floating base. Then, the configuration of the floating base is calculated according to
\begin{equation}
	\begin{split}
		\boldsymbol{M} =& \operatorname{exp}\left(q_{1}\boldsymbol{e}_{23} + q_{2}\boldsymbol{e}_{31} + q_{3}\boldsymbol{e}_{12} \right.\\ &\left.+ q_{4}\boldsymbol{e}_{01} + q_{5}\boldsymbol{e}_{02} + q_{6}\boldsymbol{e}_{03}\right)
	\end{split}
\end{equation}
Further more, the velocity and acceleration coordinates of the floating base are chosen such that
\begin{align}
	\begin{split}
		\boldsymbol{V} = &\dot{q}_{1}\boldsymbol{e}_{23} + \dot{q}_{2}\boldsymbol{e}_{31} + \dot{q}_{3}\boldsymbol{e}_{12} \\
		&+ \dot{q}_{4}\boldsymbol{e}_{01} + \dot{q}_{5}\boldsymbol{e}_{02} + \dot{q}_{6}\boldsymbol{e}_{03}
	\end{split}\\
	\begin{split}
		\dot{\boldsymbol{V}} = &\ddot{q}_{1}\boldsymbol{e}_{23} + \ddot{q}_{2}\boldsymbol{e}_{31} + \ddot{q}_{3}\boldsymbol{e}_{12} \\
		&+ \ddot{q}_{4}\boldsymbol{e}_{01} + \ddot{q}_{5}\boldsymbol{e}_{02} + \ddot{q}_{6}\boldsymbol{e}_{03}
	\end{split}
\end{align}
Based on the above coordinates, 100 samples of $\underline{q}_a, \dot{\underline{q}}_a, \ddot{\underline{q}}_a$ are uniformly generated on the interval $(0,1)$. The resultant numerical regressor $\underline{\mathbb{Y}}$ is a $1800 \times 130$ matrix. Then, the QR decomposition method is applied to validate the output of DRNG. The numerical results show that the rank of $\underline{\mathbb{Y}}$ is 94, i.e. 36 inertial parameters can not be identified independently. It is consistent with the output of our DRNG algorithm. Besides, the criteria \eqref{eq: criteria} is also satisfied, validating the correctness of the DRNG algorithm. 

As for the efficiency, the numerical method costs $144.10$ms, but the proposed algorithm costs only $1.52$ms. The case of floating-base in \cite{Wensing2024Geometric} is the Cheetah 3 quadruped, but only one leg is considered. The Unitree Go2 is re-modeled so that the RPNA can be applied. The results show that the dimension of the nullspace when the motors are omitted and the floating-base is considered is 36, which is the same as the output of the proposed DRNG algorithm. However, it takes $59.96$ ms for RPNA, which is much slower than the proposed algorithm.

\begin{remark}
	A famous result in the dynamics identification of floating-base robots is that with proper excitation and measurement of wrenches acting on floating base, all the base parameters can be identified (\cite{Ayusawa2014LeggedMSD}). We can obtain the same conclusion directly from \eqref{eq: basic coeff}. Because the three principles are independent to actuation, the problem is whether the limited measurement of actuation causes new nullspace basis vector. 
	It is noted that every rigid body in the system is involved with the configuration of the floating base. It means that $\boldsymbol{L}_{[1:6]}^i$ in \eqref{eq: basic coeff} can not remain zero under proper excitation for all $i$. Besides, no extra geometric constraints make tetrahedral points generate lines crossing $\boldsymbol{L}_{[1:6]}^i$. Therefore, no extra nullspace basis vector exists, and all base parameters can be identified from the measurement of wrenches acting on the floating-base.
\end{remark}
\subsection{PKM: 2RRU-1RRS}
\begin{figure}[t]
	\centering
	\includegraphics[width=0.9\linewidth]{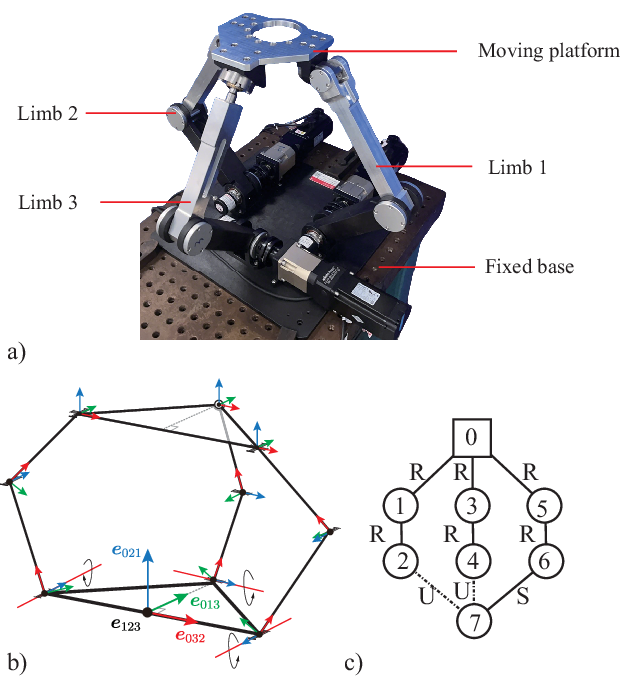}
	\caption{a). The picture of 2RRU-1RRS PKM. b) The illustration of the initial configuration. c). The graphic topology of the PKM.}
	\label{fig: 2RRU-1RRS}
\end{figure}
A main contribution of this paper is to enable the base parameter analysis of robot with multi-DOF joints in closed loops. In order to demonstrate this contribution, we apply the DRNG algorithm on a parallel kinematics mechanism (PKM) of type 2RRU-1RRS (\cite{Xu2023JMR}). Its moving platform is connected to the fixed base through three chains. Two of them is composed of R, R and U joints from base to platform and one of them is composed of R, R, and S joints. The rotational constraints in the U joints have no effect on the motion of the mechanism. Therefore, the PKM has 3 DOFs. Because the axis of the two R joints in each chain are parallel, the second rigid body of each chain is constrained such that only planar rotations are admitted. Its initial configuration and joint frames are illustrated in Fig. \ref{fig: 2RRU-1RRS}.

According to the shared points principle, there are 6 R joints, 2 U joints, and 1 S joints in the robot. They cause $3 \times 6 + 1 \times 3 = 21$ basis vectors of $Null(Y)$. Then, three links are constrained to the base through R joints, which results in $2\times 3 = 6$ extra basis vectors. Finally, six links are constrained to planar rotations, causing $2 \times 3 + 4 \times 3 = 18$ extra basis vectors based on the planar rotations principle. Therefore, there are $21 + 6 + 18 = 45$ basis vectors of $Null(Y)$, only $25$ of the $70$ inertial parameters can be identified independently. With the DRNG algorithm, the coordinates of the 45 basis vectors are found.

In the numerical validation, $5000$ samples of $\underline{q}_a$ are uniformly generated within the 3-dimensional region defined by the interval $(\frac{\pi}{9}, \frac{7\pi}{18}) \times (\frac{11 \pi}{18}, \frac{8\pi}{9}) \times (\frac{\pi}{9}, \frac{7\pi}{18})$, and $5000$ samples of $\dot{\underline{q}}_a, \ddot{\underline{q}}_a$ are uniformly generated within the interval $(0,1)$. Because of the kinematics constraints caused by closed loops, some samples $\underline{q}_a$ may cause the geometric constraints $C(\underline{q})=\underline{0}$ unsolvable, i.e. out of workspace. Besides, numerical issues occur in the kinematics calculations if the condition number of $\underline{J}_p$ in \eqref{eq: jacobian line} is too large. Therefore, only the samples $\underline{q}_a$ that lie in the workspace with the condition number of $\underline{J}_p$ is less than $30$ are selected as valid samples to calculate the numerical regressor $\underline{\mathbb{Y}}$. In the numerical experiments, only 831 of the 5000 samples are valid. The size of $\underline{\mathbb{Y}}$ is $2493 \times 70$. According to the QR decomposition method, its rank is $27$, which is not consistent with the results of our method. With a detailed observation of the diagonal elements of the R matrix, it is found that the $26$-th and $27$-th elements are both numerically small, approximately $1\times 10^{-6}$. We conclude that in the case of parallel robots, the QR decomposition method failed. For the output of our DRNG algorithm, it is found that $\|\underline{\mathbb{Y}}\underline{B}_{null}\|_{F}=4.04\times 10^{-8}$, satisfying the criteria \eqref{eq: criteria} in the numerical sense.

Besides, it takes $46.36$s to implement the numerical method, because all the samples need to be checked and the dynamics regressor is time-consuming. On the other hand, it takes only $1.81$ms to implement DRNG algorithm, which is much more faster. Although the numerical method can be accelerated by using less samples, it may cause numerical issues and wrong results.
\subsection{PKM: 2PRS-1PSR}
\begin{figure}[t]
	\centering
	\includegraphics[width=0.9\linewidth]{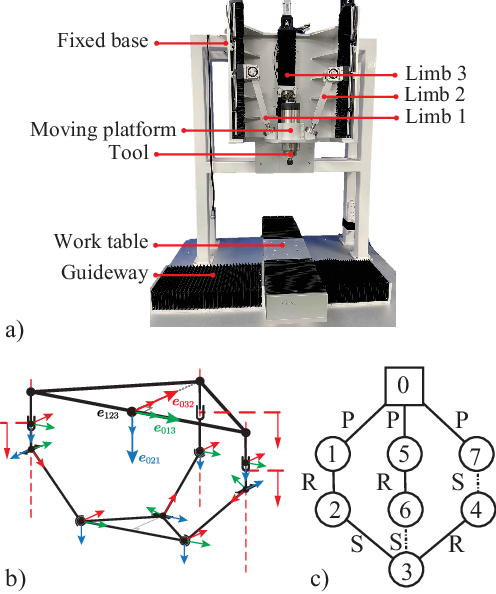}
	\caption{a). The picture of 2PRS-1PSR PKM. b) The illustration of the initial configuration. c). The graphic topology of the PKM.}
	\label{fig: 2PRS-1PSR}
\end{figure}
We further demonstrate the DRNG algorithm on a 2PRS-1PSR PKM (\cite{Shi2023Local}). It is also a 3-DOF PKM, but involved P joints. Besides, the joint couple ``PR" in the robot provides another case when planar rotations principle take effect. Similarly, this PKM has three chains, connecting the moving platform to the fixed base, resulting in a 3-DOF mechanism. Actuators are installed in the three P joints, and the corresponding translation coordinates are chosen as $\underline{q}_a$. Its initial configuration and the graphic topology are illustrated in Fig. \ref{fig: 2PRS-1PSR}.

Again, we first apply the three principles to analyze its dynamics regressor nullspace. With the principle 1, $6\times 3 + 3 \times 3 + 1 \times 3 = 30$ nullspace basis vectors are generated because of the shared points of 3 P joints, 3 R joints, and 3 S joints. With principle 2, $3 \times 3 = 9$ extra nullspace basis vectors are generated. As about principle 3, it is found that the rigid bodies are constrained to follow planar rotations because of the ``PR" joints. As a result, $4 \times 2=8$ extra basis vectors exist. Therefore, the dynamics regressor nullspace has $30+9+8=47$ independent basis vectors, which makes only $70-47=23$ inertial parameters can be identified independently.

The numerical validation procedure is the same as that of the 2RRU-1RRS PKM. $5000$ samples of $\underline{q}_a$ are first uniformly generated in the region $(-0.15,0.09) \times (-1.33,0.10) \times (-0.10,0.13)$, and only the ones in workspace with the condition number of $\underline{J}_p$ less than $30$ are chosen as valid samples. Correspondingly, $\dot{\underline{q}}_a, \ddot{\underline{q}}_a$ are uniformly generated within the interval $(0,1)$. The resultant numerical regressor $\underline{\mathbb{Y}}$ is with the size $8793\times 70$. The rank of $\underline{\mathbb{Y}}$ is $23$, which is identical to the analysis based on the proposed principles. Furthermore, the DRNG algorithm is applied to this PKM, and the output satisfies the criteria in \eqref{eq: criteria}. For the same reason as the 2RRU-1RRS PKM, the DRNG algorithm costs only $2.26$ms, while the numerical method costs $41.79$s.

\section{Conclusion and future work}
\label{sec: conclusion}
In this paper, the robot dynamics is reformulated with PGA, and a novel geometric method is proposed to determine base parameters in robot inertia identification. With four points that can form a tetrahedron, the coefficients in the dynamics regressor matrix are formulated as the join of these points' positions and their accelerations. The identification model of a robot is then constructed base on this observation. Furthermore, three principles, including shared points principle, fixed points principle and planar rotations principle, are proposed to determine the base parameters analytically. The dynamics regressor nullspace generator (DRNG) algorithm is then developed to implement the three principles. It takes the geometry model of a robot (fixed-base or floating-base, with or without closed-loops), and returns the basis vectors of dynamics regressor nullspace. Its complexity is $O(1)$ in theory. Through comprehensive demonstrations, the correctness of the algorithm is validated, and the method is proved to be general, robust, and efficient.

In future work, the method will be further applied to accelerate the dynamics computations, and it will be generalized to the under-actuated systems.

\begin{acks}
	This work was supported by the National Natural Science Foundation of China (Grant No.  52275501, 51935010) and the State Key Laboratory of Mechanical System and Vibration (Grant No. MSVZD202503).
\end{acks}

\bibliographystyle{SageH}

\end{document}